\pdfoutput=1
\documentclass[11pt]{article}

\usepackage[final]{acl}

\usepackage{times}
\usepackage{latexsym}

\usepackage[T1]{fontenc}
\usepackage[utf8]{inputenc}

\usepackage{microtype}

\usepackage{inconsolata}

\usepackage{graphicx}

\usepackage{enumitem}
\usepackage{wrapfig}
\usepackage[utf8]{inputenc} % allow utf-8 input
\usepackage[T1]{fontenc}    % use 8-bit T1 fonts
\usepackage{hyperref}       % hyperlinks
\usepackage{url}            % simple URL typesetting
\usepackage{booktabs}       % professional-quality tables
\usepackage{amsfonts}       % blackboard math symbols
\usepackage{nicefrac}       % compact symbols for 1/2, etc.
\usepackage{microtype}      % microtypography
\usepackage{xcolor}         % colors
\usepackage{tcolorbox} 
\usepackage{pifont} 
\usepackage{amsthm}  % 提供数学定理环境
\usepackage{amsmath}
\usepackage{cuted} 
\usepackage{bm}
\usepackage{amsmath, amssymb, amsthm}
\newtheorem{proposition}{Proposition}

\usepackage{algorithm} 
\usepackage{algpseudocode} 
\usepackage{caption}
\usepackage{subcaption}

\usepackage{pifont} 
\usepackage{xcolor}
\definecolor{deepgreen}{rgb}{0.0, 0.392, 0.0}
\providecommand{\R}{\mathbb{R}}

\usepackage{mathtools}
\newcommand{\eqqcolon}{\coloneqq}
\theoremstyle{plain}

\title{BSFA: Leveraging the Subspace Dichotomy to Accelerate Neural Network Training}

\author{
  Wenjie Zhou\textsuperscript{\rm 1,2}\thanks{Equal contribution},\;
  Bohan Wang\textsuperscript{\rm 3}\footnotemark[1],\;
  Wei Chen\textsuperscript{\rm 1,2}\thanks{Corresponding author},\;
  Xueqi Cheng\textsuperscript{\rm 1,2}\\
  \textsuperscript{\rm 1}State Key Laboratory of AI Safety, Institute of Computing Technology,\\  Chinese Academy of Sciences \\
  \textsuperscript{\rm 2}University of Chinese Academy of Sciences\\
  \textsuperscript{\rm 3}Alibaba Group\\
  \texttt{zj4323005@gmail.com, bhwangfy@gmail.com, \{chenwei2022, cxq\}@ict.ac.cn}
}

\begin{document}
\maketitle
\maketitle

\begin{abstract}
Recent studies \citep{gur2018gradient,song2024does, wen2024understanding} highlight a fundamental dichotomy in deep learning optimization: 
Although parameter updates along the top eigendirections of the loss Hessian (Dom-space) capture most of the update magnitude, they often contribute minimally to loss reduction. In contrast, updates in the orthogonal component (Bulk-space) have smaller magnitudes but drive most learning progress.
In this work, we further advance the understanding of this phenomenon and introduce the \textbf{Bulk-Space-Filtration-Accelerator (BSFA)}, a novel plug-and-play framework. BSFA accelerates training by differentially scaling update components projected onto these distinct subspaces, simultaneously enhancing stability by moderating updates in the dominant subspace and boosting convergence speed by amplifying those in the bulk-space.
To ensure BSFA is both practical and scalable for contemporary large models, we introduce two key innovations: an efficient estimator using Principal Component Analysis (PCA) on historical updates for fast subspace estimation, and a block-wise strategy that applies this estimation on a per-parameter-block basis. These designs make BSFA computationally tractable and highly effective.
We demonstrate BSFA's acceleration across various tasks, notably achieving approximately 2$\times$ speedup when pre-training LLaMA-72M on WikiText-103 and LLaMA-134M on OpenWebText compared to vanilla AdamW.

\end{abstract}    
% \section{TODO List}
% {
% \color{red}
% 1. Move the part of the variance of gradient into the last section of large-scale experiment as evidence

% }

\section{Introduction}
\label{sec:intro}

Deep learning has revolutionized artificial intelligence, achieving remarkable breakthroughs across domains such as computer vision \citep{he2016deep,dosovitskiy2020image,liu2021swin}, natural language processing (especially LLMs) \citep{vaswani2017attention,devlin2018bert,achiam2023gpt}, and healthcare \citep{miotto2018deep,liu2020deep,esteva2019guide}. At the heart of this success lies optimization - the indispensable engine driving the training of complex neural networks. By iteratively minimizing loss functions through algorithms like stochastic gradient descent (SGD) \citep{robbins1951stochastic} and its adaptive variants \citep{duchi2011adaptive,tieleman2012lecture,kingma2014adam}, optimization enables models to uncover intricate patterns in high-dimensional data. However, the ever-increasing scale of data and model complexity \citep{kaplan2020scaling,hoffmann2022training,molybog2023theory} has led to escalating training costs, motivating researchers to continuously innovate toward methods that can achieve high computational efficiency with reduced resource demands. This relentless pursuit of scalable and adaptive optimization techniques underscores their critical role in advancing the frontiers of deep learning.

Recent research has revealed that optimization landscapes in deep learning exhibit unique characteristics distinct from traditional machine learning objectives. For instance, \citet{zhang2020gradient} demonstrated that neural network training processes often exhibit dramatic variations in local smoothness positively correlated with gradient magnitudes. \citet{wu2018sgd,cohen2021gradient} identified the "edge of stability" phenomenon, where Hessian matrices implicitly adapt to optimization hyperparameters until oscillation emerges. \citet{zhang2024why} further discovered "Block Heterogeneity" in Transformer training, characterized by significant disparities in Hessian spectra across parameter blocks. These observations not only explain the superior performance of adaptive optimizers in certain architectures but also inspire specialized optimization algorithm designs \citep{roulet2024stepping,adam_mini_2024}.

In overparameterized neural networks, the loss landscape is highly anisotropic. Spectral analyses of the Hessian reveal two distinct components. The bulk of its eigenvalues lies near zero and corresponds to flat directions in parameter space, while a small set of large eigenvalues identifies sharp directions \citep{hochreiter1997flat, sagun2017empirical}. Extensive research \citep{pmlr-v97-ghorbani19b, Papyan2018TheFS} into the Hessian's prevalent low-rank structure further quantifies this: for $k$--class classification problems, there are typically around $k$ such dominant eigenvalues, while for Large Language Models \citep{zhang2024why, liu2023sophia}, the largest few tens of eigenvalues are generally significantly more prominent and are sufficient to define the dominant subspace. As a result, at each point in parameter space, the tangent space splits into a high-dimensional flat subspace (\textit{bulk subspace}) and a low-dimensional sharp subspace(\textit{dominant subspace}).

% 这里的改动是想在Intro里面简要说明什么是Dichotomy
 Building on this geometric picture, most recently, a line of works \citep{gur2018gradient, song2024does, wen2024understanding} show that optimization dynamics in the dominant and bulk subspaces follow a clear dichotomy. Most update norm falls into the dominant subspace but contributes little to loss reduction. In contrast, the relatively small update projected onto the bulk subspace drives most of the loss descent (see Section \ref{sec:part1} for details). This discovery has been leveraged to explain the effectiveness of modern learning rate schedulers like Warm-Stable-Decay (WSD) \citep{hu2024minicpm}.

Building upon these observations, this work moves beyond phenomenological explanations to algorithmic innovation. Specifically, we address the fundamental question:
\begin{center}
\textit{Can we exploit the subspace structure in update to accelerate optimization convergence?}    
\end{center}

Our contributions are threefold:
\begin{itemize}
    \item [1.] We first extend the understanding of the subspace dichotomy in deep learning optimization. Our empirical analysis reveals that independently modulating update components within the dominant and bulk subspaces—derived from the Hessian Eigenspectrum—yields distinct effects: controlling update magnitudes in the dominant subspace primarily affects training stability, while updates within the bulk subspace predominantly influence convergence speed.
    
    %moderating dominant subspace updates enhances training stability and mitigates loss spikes, while amplifying bulk subspace updates accelerates convergence.
    \item [2.] Based on this insight, we propose the \textbf{Bulk-Space-Filtration-Accelerator (BSFA)}, a plug-and-play framework that differentially scales updates in these subspaces. Initial validation with exact Hessian information shows up to 4$\times$ ~ acceleration on training ResNet18 on CIFAR10 and DenseNet121 on CIFAR100. Then we introduce two key algorithmic enhancements to make BSFA practical and scalable: an efficient PCA-based estimator using historical updates for rapid dominant subspace approximation and a block-wise strategy applying this per parameter block. These render BSFA computationally efficient for large models.
    \item [3.] We validate the practical BSFA framework on large-scale Transformer models, demonstrating significant training acceleration. Notably, BSFA achieves approximately 2$\times$ speedup in pre-training LLaMA-72M on WikiText-103 and LLaMA-134M on OpenWebText, and in training ViT-Small on ImageNet-1k, compared to AdamW.
\end{itemize}

\section{Related work}
\label{sec:related}

In this section, we review existing literature on optimization behaviors unique to deep learning and acceleration techniques in this domain.

\paragraph{\textbf{Optimization behaviors specific to deep learning.}}
While classical machine learning tasks \cite{platt1998sequential,freund1996experiments} typically feature convex, smooth landscapes with static local properties, deep learning optimization landscapes are inherently more complex. These landscapes are often characterized by non-convexity \cite{li2018visualizing} and non-smoothness \cite{zhang2020gradient}, creating seemingly chaotic optimization dynamics. Recent studies, however, have revealed intriguing structural patterns within this complexity. For instance, \cite{zhang2020gradient} observed in NLP tasks that the spectral norm of the Hessian matrix exhibits a positive correlation with gradient norms, providing theoretical justification for the effectiveness of gradient clipping techniques. In transformer-based architectures, \cite{zhang2024why} identified heavy-tailed distributions in parameter updates, while \cite{pmlr-v235-zhu24h} demonstrated that large learning rates induce oscillations in subtle classification rule learning while preserving deterministic feature acquisition. 
A particularly influential observation across multiple studies \cite{jastrzebski2018width,wu2018sgd, jastrzebski2020break,cohen2021gradient} is the "edge of stability" phenomenon, where loss sharpness increases until reaching an oscillation threshold determined by optimization hyperparameters.

Most relevant to our work are the findings of \cite{song2024does, wen2024understanding}, who demonstrated that parameter updates in deep learning can be decomposed into two distinct subspaces: one accounting for the majority of update magnitude but minimal loss reduction, and another comprising smaller updates that drive most of the loss descent. This dichotomy forms the foundation for our proposed acceleration methodology.

\paragraph{\textbf{Acceleration techniques in deep learning.}} Since the introduction of stochastic gradient descent (SGD) \cite{robbins1951stochastic}, researchers have persistently sought improvements to optimization efficiency. Early innovations like momentum \cite{polyak1964some} enhanced convergence by incorporating historical gradient information. The advent of adaptive learning rate methods marked a pivotal advancement, with algorithms like AdaGrad \cite{duchi2011adaptive}, RMSProp \cite{tieleman2012lecture}, and Adam \cite{kingma2014adam} addressing scale variations across parameters. Notably, Adam's dominance in modern practice stems from its synthesis of momentum and adaptive step-size mechanisms. Recent efforts have explored more radical departures, including sign-based updates \cite{chen2023symbolic}, matrix-structured optimization \cite{jordan2024muon}, and second-order approximations \cite{liu2023sophia}, each targeting different aspects of the optimization geometry.

%In contrast to these approaches, our methodology draws inspiration from the intrinsic structure of parameter updates. By explicitly leveraging the subspace decomposition phenomenon observed in \cite{song2024does, wen2024understanding}, we propose a novel acceleration framework that strategically prioritizes critical update components.

In contrast to these approaches, our methodology draws inspiration from the intrinsic structure of parameter updates. By explicitly leveraging the subspace decomposition phenomenon observed in \cite{song2024does}, we propose a novel acceleration framework that strategically prioritizes critical update components. 
Relatedly, IRE \cite{NEURIPS2024_d712c862} accelerates implicit regularization by separating flat vs.\ sharp components using a diagonal Fisher/Gauss--Newton estimate and boosting the effective step size only on the flatter coordinates. 
Meanwhile, Blockwise LR \cite{wang2025sharpnessdisparityprincipletransformers} exploits sharpness disparities across Transformer block types and assigns larger learning rates to flatter blocks while keeping the sharpest (e.g., Norm) at the base LR; in contrast, our method uses the Hessian-eigenvector structure observed by \cite{song2024does} more explicitly rather than diagonal/blockwise heuristics.    
%\input{sec/3_Preliminary}    
%\section{BSFA: Accelerating Training through Amplify Bulk Update}
\section{Subspace Dichotomy in Training Dynamics}
%or Leveraging Subspace Differentiation in Optimization, Exploiting the Split of Dominant and Bulk Subspaces in Training

\label{sec:part1}

In this section, we first briefly introduce the subspace dichotomy phenomenon \citep{gur2018gradient, song2024does, wen2024understanding} and then present our main findings.
% 就直接简短一下就行，就是detailed introduction而不是具体是notation

 Let \(L(\theta):\mathbb{R}^p\to\mathbb{R}\) be the loss function of the neural network parameterized by $\theta$. For \(\theta\in\mathbb{R}^p\), let \(H(\theta)=\nabla^2L(\theta)\in\mathbb{R}^{p\times p}\), with eigenvalues \(\lambda_1(\theta)\ge\cdots\ge\lambda_p(\theta)\) and corresponding orthonormal eigenvectors \(u_1(\theta),\dots,u_p(\theta)\). With these notations, \cite{gur2018gradient, song2024does} defines the top-\(k\) \textbf{dominant subspace} as \(S_k(\theta)=\mathrm{span}\{u_1(\theta),\dots,u_k(\theta)\}\), and the \textbf{bulk subspace} as its orthogonal complement \(S_k^\perp(\theta)\). The corresponding projectors to the top-\(k\) \text{dominant subspace} and the \text{bulk subspace} are respectively denoted as $P_k(\theta)$ and $P_k^\perp(\theta)$.
 %Next, we introduce the concepts of subspace and subspace projector.
%  \vspace{-5pt}
% \begin{itemize}[leftmargin=10pt]
%   \item \textit{Dominant and Bulk subspaces.} For integer \(k\), define \(S_k(\theta)=\mathrm{span}\{u_1(\theta),\dots,u_k(\theta)\}\) as the top-\(k\) \textbf{dominant subspace} at $\theta$ and its orthogonal complement \(S_k^\perp(\theta)\) as the \textbf{bulk subspace}.  
%   %Unless otherwise specified, \(k\) is taken as the number of classes for classification tasks, and \(k=50\) serves as a sufficient estimate for LLMs.
% \vspace{-5pt}
%   \item \textit{Projectors.} Let \(P_k(\theta)=\sum_{i=1}^k u_i(\theta)u_i(\theta)^\top\) as the projector onto \(S_k(\theta)\) and \(P_k^\perp(\theta)=I-P_k(\theta)\) as the projector onto \(S_k^\perp(\theta)\).
% \end{itemize}
% \vspace{-5pt}

The key finding of   \citet{song2024does} is \textit{learning happens in the Bulk Subspace:} Consider a SGD training process \(\theta_{t+1}=\theta_t-\eta g_t\) (where $g_t$ denotes gradient at step $t$). If each update is projected onto the dominant subspace, i.e. \(\theta_{t+1}=\theta_t-\eta P_k(\theta_t)g_t\), it fails to decrease the training loss further. Conversely, projecting each update onto the bulk subspace, \(\theta_{t+1}=\theta_t-\eta P_k^\perp(\theta_t)g_t\), is still capable of driving down the training loss. We validate this observation in Figure~\ref{fig:minhak}, where we train a 2-layer Transformer (2M parameters) on the SST2 dataset using SGD; details are provided in Appendix~\ref{appn:par1}.

\begin{figure}[htbp]
  \includegraphics[width=\columnwidth]{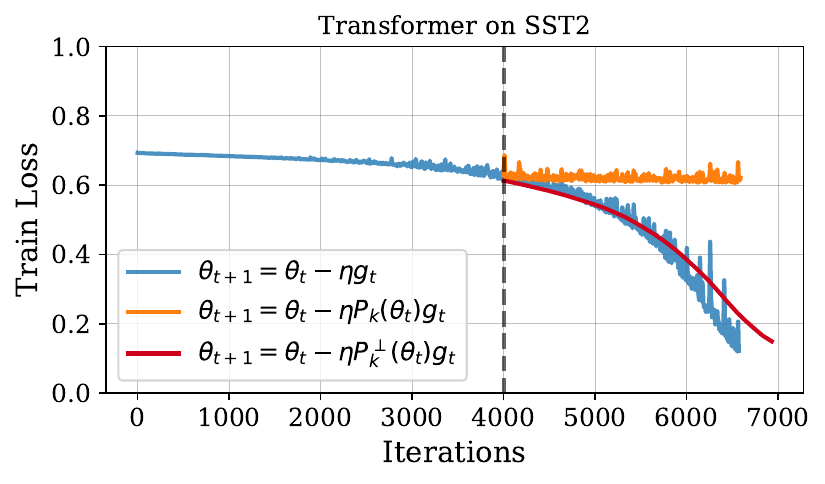}
  \caption{SGD in dominant and bulk subspace ($k=2$).}
    \label{fig:minhak}
\end{figure}

%In Figure~\ref{fig:minhak}, we train a 2-layer Transformer(2M parameters) on the SST2 dataset using SGD (Details are provided in Appendix~\ref{appn:par1}) to display the above main result of \cite{song2024does}. The figure demonstrates that updates projected onto the bulk subspace are still capable of driving down the training loss, whereas those restricted to the dominant subspace fail to do so.

\begin{figure*}[t]
  \centering
  \begin{subfigure}[t]{0.48\linewidth}
    \includegraphics[width=\linewidth]{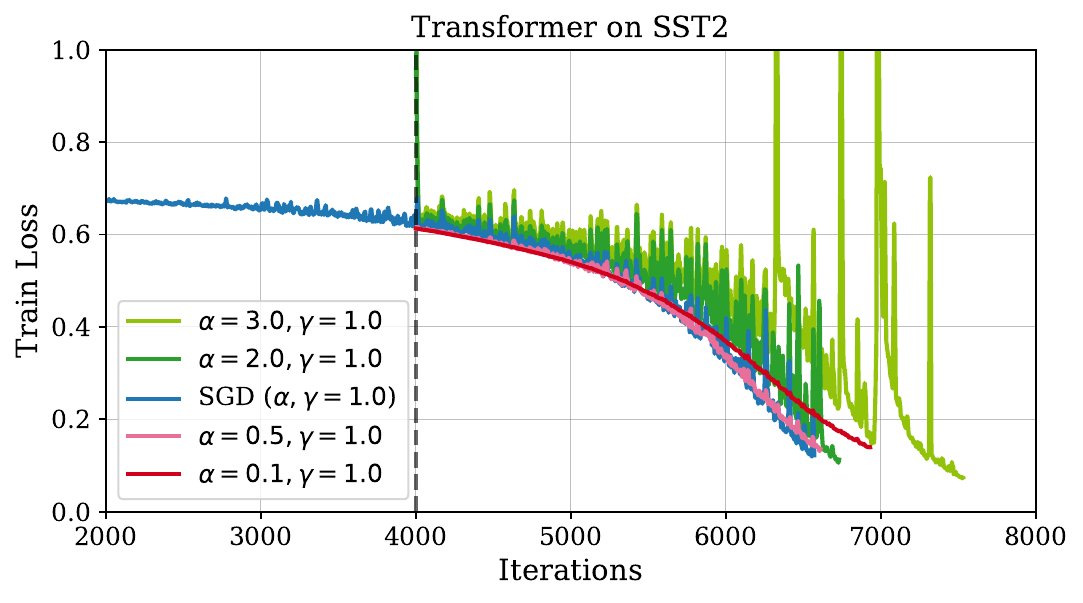}
    \caption{Varying the dominant subspace update magnitude.}
    \label{fig:demonstra_a}
  \end{subfigure}
  \hfill
  \begin{subfigure}[t]{0.48\linewidth}
    \includegraphics[width=\linewidth]{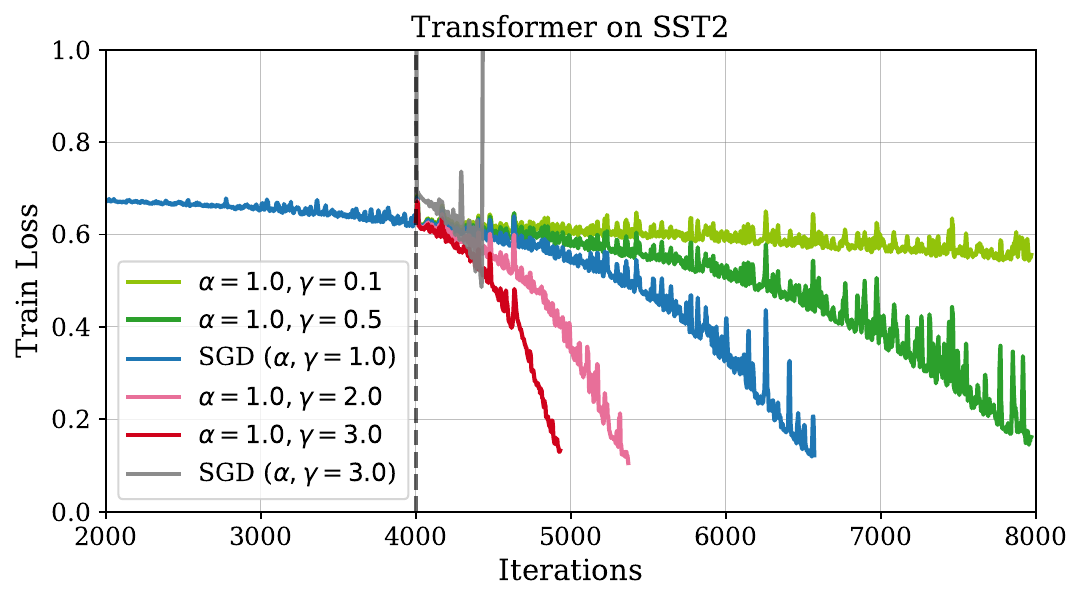}
    \caption{Varying the bulk subspace update magnitude.}
    \label{fig:demonstra_b}
  \end{subfigure}
  \caption{\textbf{Subspace‐specific Update Scaling.} We follow the setting of Figure~\ref{fig:minhak}(SGD), and after 4000 steps, we introduce two scaling factors, $\alpha$ and $\gamma$, to independently modulate the update magnitudes in the dominant and bulk subspaces. Training is terminated once the model's training accuracy reaches 0.99.}
  \label{fig:demonstra}
\end{figure*}

Building on this observation, a natural question arises: how do updates projected onto these two subspaces affect the training process? To tease apart these two effects, we introduce a simple, two-parameter projector \(\mathcal{P}_{\alpha,\gamma}(\theta)\) that independently scales the updates in the dominant and bulk subspaces. Specifically, let
\[
\mathcal{P}_{\alpha,\gamma}(\theta) \;=\; \alpha\,P_k(\theta)\;+\;\gamma\,P_k^\perp(\theta),
\]
where \(P_k(\theta)\) is the projector onto the dominant subspace and \(P_k^\perp(\theta)=I-P_k(\theta)\), as defined above. Then, we conduct experiments using this projector to modify the training iteration: \(
\theta_{t+1} = \theta_t - \eta\,\mathcal{P}_{\alpha,\gamma}(\theta_t)\,g_t,\) scaling the dominant and bulk components of the update by \(\alpha\) and \(\gamma\). Following the setup of Figure~\ref{fig:minhak}, we sweep \(\alpha\in\{0.1,0.5,1,2,3\}\) and \(\gamma\in\{0.1,0.5,1,2,3\}\) at step $4000$, we then compare these projected training runs with a well-tuned SGD baseline, and the results are presented in Figure~\ref{fig:demonstra}. We summarize two notable observations as follows:
%and report the resulting training dynamics in Figure~\ref{fig:demonstra}.   Based on these results, 
\vspace{-5pt}
\begin{itemize}[leftmargin=10pt]
  \item \textbf{Dom‐Update and Training Stability:} In Figure~\ref{fig:demonstra_a}, with the bulk subspace learning rate held constant, varying the dominant subspace scaling factor $\alpha$ does not significantly alter the overall convergence rate. However, larger values of $\alpha$ (\textcolor{deepgreen}{green curves}) induce more frequent and pronounced loss spikes during training and induce instability, whereas smaller values of $\alpha$ (\textcolor{red}{red curves}) mitigate these spikes and promote smoother convergence.
  \vspace{-5pt}
  \item \textbf{Bulk‐Update and Training Speed:} In Figure~\ref{fig:demonstra_b}, with the dominant subspace learning rate fixed, varying the bulk subspace scaling factor $\gamma$ has a marked impact on convergence speed. When chosen appropriately, higher values of $\gamma=3$ (\textcolor{red}{red curves}) can substantially accelerate training; however, simply increasing the overall update by 3 times (which is equal to setting $\alpha,\gamma=3$) (\textcolor{gray}{grey curve}) would lead to divergence.
\end{itemize}
\vspace{-5pt}
%From our analysis above, updates in the dominant and bulk subspaces exert qualitatively different influences on the training dynamics. Specifically, increasing the magnitude of updates in the dominant subspace tends to destabilize the optimization trajectory, whereas amplifying updates in the bulk subspace directly accelerates convergence, yielding faster loss reduction.
These findings highlight a functional dichotomy: updates in the dominant subspace predominantly govern training stability, with smaller magnitudes promoting smoother convergence. Conversely, updates in the bulk subspace are primary drivers of convergence speed, where appropriately scaled larger magnitudes can yield substantial acceleration. Such distinct behaviors strongly motivate a decoupled control of update components within these respective subspaces to enhance overall training performance.
 
\section{BSFA: Accelerating Training through Amplify Bulk Update}
\label{sec:part2}

This section details our methodology for leveraging the observed subspace dichotomy to accelerate neural network training. We first introduce the Bulk‐Space‐Filtration‐Accelerator (BSFA) framework in Section \ref{sec:bsfa_frame}. Subsequently, Section \ref{sec:pca} presents an efficient PCA‐based estimator for approximating dominant eigendirections. Finally, Section \ref{sec:block-wise} describes a block‐wise BSFA strategy to enhance scalability for large models by exploiting the Hessian's approximate block‐diagonal structure.

\subsection{BSFA Framework}
\label{sec:bsfa_frame}

\begin{figure*}[htbp]
  \centering
  \begin{subfigure}[t]{0.24\linewidth}
    \centering
    \includegraphics[width=\linewidth]{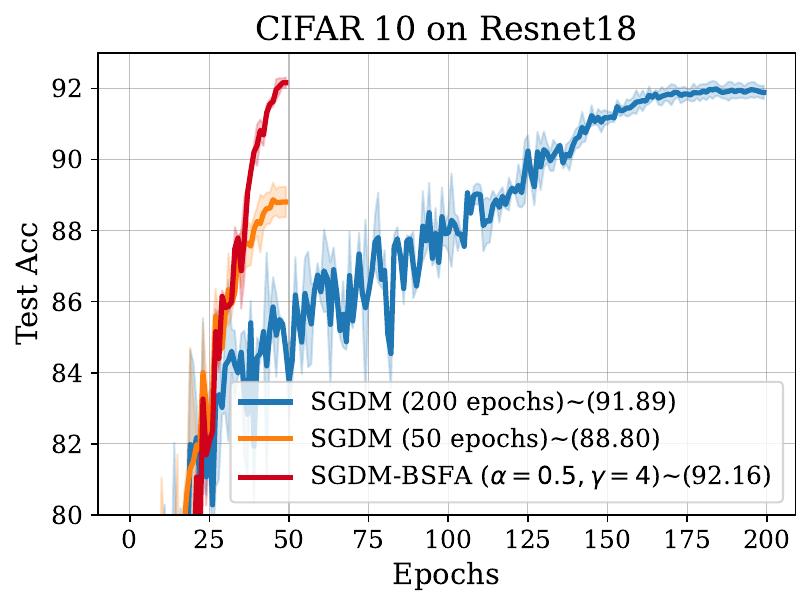}
    \caption{CIFAR10 Accuracy}
    \label{fig:cifar10}
  \end{subfigure}
  \hfill
  \begin{subfigure}[t]{0.24\linewidth}
    \centering
    \includegraphics[width=\linewidth]{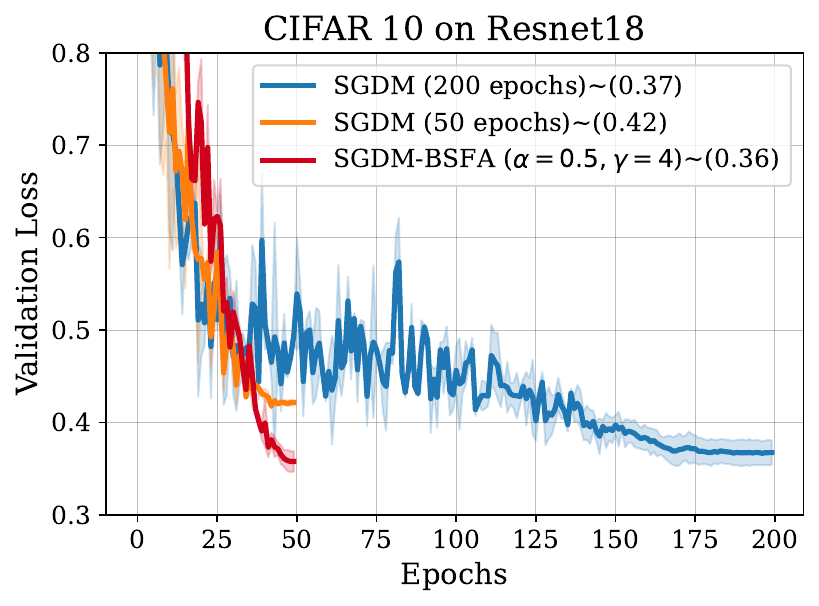}
    \caption{CIFAR10 Validation Loss}
    \label{fig:cifar10_tloss}
  \end{subfigure}
  \hfill
  \begin{subfigure}[t]{0.24\linewidth}
    \centering
    \includegraphics[width=\linewidth]{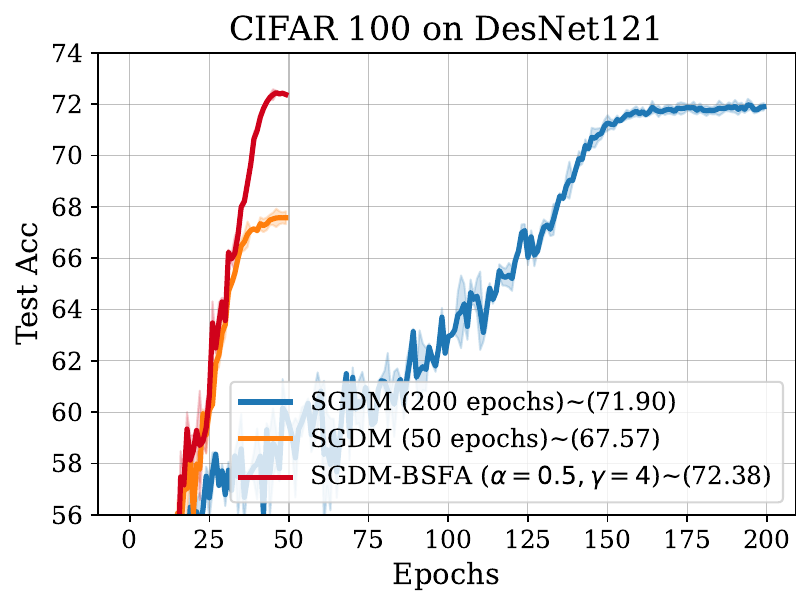}
    \caption{CIFAR100 Accuracy}
    \label{fig:cifar100}
  \end{subfigure}
  \hfill
  \begin{subfigure}[t]{0.24\linewidth}
    \centering
    \includegraphics[width=\linewidth]{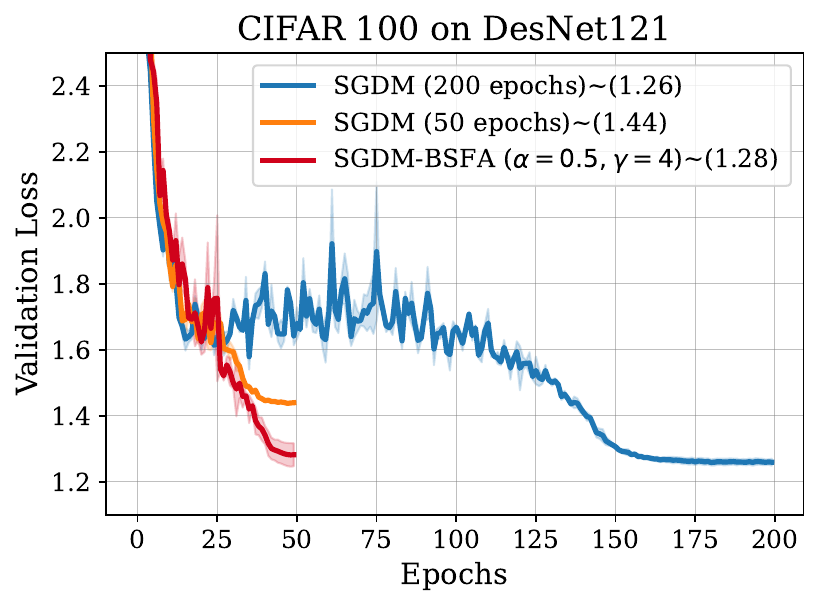}
    \caption{CIFAR100 Validation Loss}
    \label{fig:cifar100_tloss}
  \end{subfigure}
  
  \caption{\textbf{SGDM with BSFA ($\alpha=0.5, \gamma=4$) achieves a $4\times$ acceleration compared to tuned SGDM.} In both experiments, we employ a cosine learning‐rate schedule with different total epoch counts. In each case, BSFA consistently achieves lower validation loss at identical epoch checkpoints and matches the baseline’s 200‐epoch test accuracy in just 50 epochs. We provide training results for three random seeds.}
  \label{fig:cifar10&cifar100}
\end{figure*}

Recalling Figure~\ref{fig:demonstra}, updates in the dominant and bulk subspaces contribute differently to the training dynamics. Building on this insight, we propose Bulk‐Space‐Filtration‐Accelerator (BSFA), a plug‐and‐play acceleration framework that can be integrated with any optimizer. Concretely, given a generic update \(
\theta_{t+1} \;=\;\theta_t + v_t,
\) BSFA modifies it to
\[
\theta_{t+1}
= \theta_t + \eta\,\mathcal{P}_{\alpha,\gamma}(\theta_t)\,v_t,
\]
where \(\mathcal{P}_{\alpha,\gamma}(\theta_t)\) denote the projector at iteration \(t\), it scales the base optimizer’s update in the dominant subspace by \(\alpha\) and in the bulk subspace by \(\gamma\). In practice, as illustrated in Figure~\ref{fig:demonstra}, we typically choose \(\alpha<1\) to enhance training stability and setting \(\gamma>1\) generally promotes a faster decrease in the loss. We recompute \(\mathcal{P}_{\alpha,\gamma}(\theta_t)\) every \(T\) steps, with \(T=10\) by default. We summarize the BSFA framework in Algorithm~\ref{algo:BSFA}.

\begin{algorithm}
\caption{Bulk-Space-Filtration-Accelerator (\textbf{BSFA})}
\label{algo:BSFA}
\begin{algorithmic}[1]
\State \textbf{Input:}  Current iteration step $t$, Update $\boldsymbol{v}_t \in \mathbb{R}^{p}$ of base optimizer at step $t$, Dominant subspace's rank $k$, Interval $T$, Projection matrix estimator, {Domspace scaler $\alpha$, Bulkspace scaler $\gamma$}.
\State BSFA Projector $\mathcal{P}_{\alpha,\gamma} \gets \boldsymbol{I}_{k \times d}$
%\State \textcolor{gray}{\# Step 1: Update BSFA projection matrix}
\If{$t \mod T = 0$ \textbf{and} $t>0$} 
    \State {$\mathcal{P}_{\alpha,\gamma} \gets \text{Estimator}(\alpha, \gamma, k)$}
\EndIf
%\State \textcolor{gray}{\# Step 2: Compute projected update}
\State $\boldsymbol{v}^{\prime}_{t} \gets \mathcal{P}_{\alpha,\gamma} \boldsymbol{v}_t$
\State \textbf{Return} $\boldsymbol{v}^{\prime}_{t}$
\end{algorithmic}
\end{algorithm}

By definition, the projector $\mathcal{P}{\alpha,\gamma}$ needs to be constructed by approximating the top $k$ Hessian eigenvectors ${u_i}{i=1}^k$. To obtain these, a common approach is the Lanczos method (see Appendix~\ref{appen:lanczos} for details).   Leveraging this off‐the‐shelf routine, we define our \textbf{Lanczos‐based Projector Estimator (LPE)} as a direct application of Lanczos to construct the following projector:

LPE is applied in Figures \ref{fig:minhak} and \ref{fig:demonstra}, yielding highly accurate estimates. Pseudocode for LPE is given in Appendix~\ref{appen:lanczos}, Algorithm~\ref{algo:lanczos}.

To validate the BSFA framework with its Lanczos-based Projector Estimator (LPE), we conducted experiments on ResNet18/CIFAR10 and DenseNet121/CIFAR100 (Figure~\ref{fig:cifar10&cifar100}) by integrating BSFA (using $\alpha=0.5, \gamma=4$) with well-tuned SGDM (further experiment details are provided in Appendix~\ref{appn:part2}). 
These experiments demonstrated that BSFA enables SGDM to achieve higher terminal test accuracy within the same number of training epochs. Furthermore, it provides a significant 4$\times$ acceleration, allowing SGDM to reach the baseline's 200-epoch test accuracy and validation loss in just 50 epochs.

\subsection{Fast Dom-Subspace Estimation via Principal Component Analysis}
\label{sec:pca}

\begin{figure*}[t]
  \centering
    \begin{subfigure}[c]{0.32\linewidth}
    \centering
    \includegraphics[width=\linewidth]{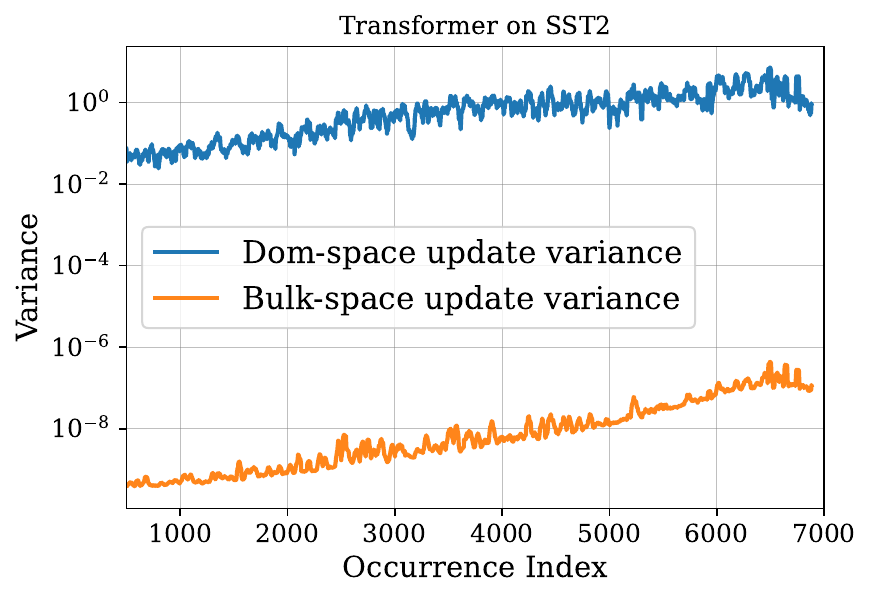}
    \caption{Variance of updates in two subspaces.}
    \label{fig:pca_var}
  \end{subfigure}
    \hfill
  \begin{subfigure}[c]{0.32\linewidth}
    \centering
    \includegraphics[width=\linewidth]{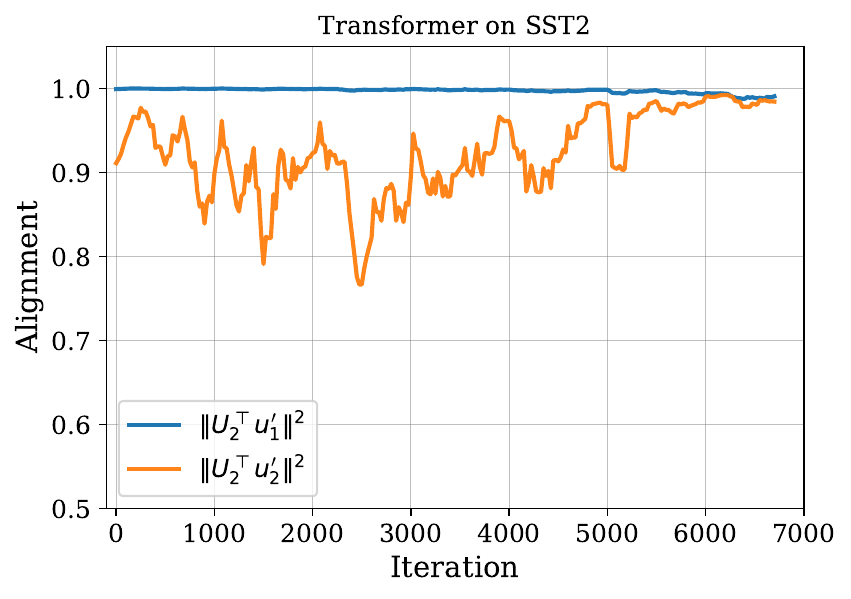}
    \caption{Squared projection norms.}
    \label{fig:pca1}
  \end{subfigure}
  \hfill
  \begin{subfigure}[c]{0.32\linewidth}
    \centering
    \includegraphics[width=\linewidth]{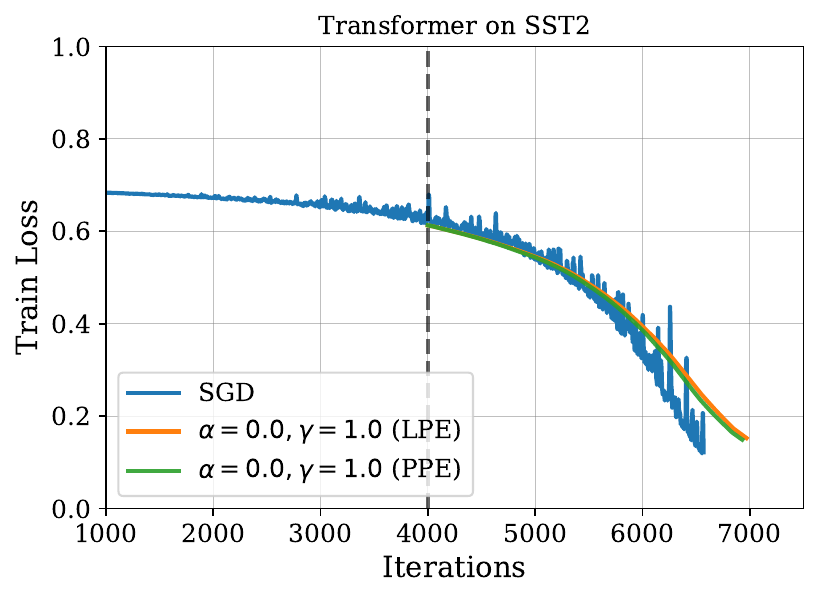}
    \caption{Training loss comparison.}
    \label{fig:pca2}
  \end{subfigure}

  \caption{(a) The Variance of updates in Dom-space is much higher than Bulk-space. (b) Alignment of Dom-space estimated by PPE and LPE is close to 1. (c) PPE and LPE exhibit comparable efficacy.}
  \label{fig:combined}
\end{figure*}

The previous section validated BSFA’s substantial acceleration benefits. However, a key challenge of LPE estimator is the high computational cost of estimating the top eigenvectors using the Lanczos method, which involves multiple forward and backward propagation steps, rendering the top-$k$ eigenvector computation both prohibitively expensive and very time-consuming. 
Therefore, we next explore how this computational time can be significantly reduced by efficiently approximating the dominant eigendirections during training.

\paragraph{\textbf{Key insight: Oscillatory Dynamics in Dominant Subspace.}}
To address this, we first draw intuition from the oscillatory dynamics of historical updates in the dominant subspaces. We present the following proposition to provide insight that PCA can capture these top eigenvectors over historical updates.
\begin{proposition}[Top eigenspace recovery via PCA]
\label{prop:pca}
Let $H\in\R^{p\times p}$ be symmetric positive–semidefinite with simple eigenvalues $\lambda_1>\dots>\lambda_k>\lambda_{k+1}$ and suppose all trailing eigenvalues are equal, i.e.\ $\lambda_{k+1}=\dots=\lambda_d\eqqcolon\lambda_{\mathrm{tail}}\ge0$.  
Pick a stepsize $\eta>0$ such that $\eta\lambda_k>1,\;0<\eta\lambda_{\mathrm{tail}}<1$, and $\eta(\lambda_k+\lambda_{\mathrm{tail}})>2$.  
Run gradient descent $\theta_{s+1}=\theta_s-\eta H\theta_s$ from any $x_0$ whose first $k$ eigencoordinates are non-zero, write gradient $g_s=H\theta_s$, and form $G_t=[g_t,\dots,g_{t+l-1}]$ with window length $l\ge k$ and $l>1$.  
Then, as $t\to\infty$, the $k$-dimensional principal subspace identified by PCA of the gradient matrix $G_t$ converges to the target eigenspace $S_k=\mathrm{span}\{v_1,\dots,v_k\}$.
\end{proposition}

Proposition \ref{prop:pca} indicate that the top-$k$ eigenvectors of loss Hessian can be recovered by applying PCA to a list of recent gradients: directions with larger eigenvalues oscillate more strongly and thus dominate the principal components (proof is in Appendix \ref{appen:proof}). In Figure \ref{fig:pca_var}, we validate this oscillatory behavior on a Transformer trained on SST2 (same setup as Figure~\ref{fig:demonstra}) and, for each update, project the most recent 30 gradient vectors onto the estimated dominant subspace and its orthogonal complement. We then compute the sample variance of these projections and observe that the Variance in the dominant subspace far exceeds that in the bulk subspace, confirming that PCA on gradient histories effectively isolates the Hessian's leading eigenspace.

Motivated by Proposition~\ref{prop:pca} and the variance separation phenomenon illustrated in Figure~\ref{fig:pca_var}, we introduce our\textbf{ PCA‐based projector estimator (PPE)} in Algorithm~\ref{algo:pca}. The estimator maintains a fixed‐length queue storing the most recent \(l\) updates(with $l>k$) and performs PCA on the queued updates, and retains the top \(k\) principal components.

\begin{algorithm}
\caption{PCA‐based Projector Estimator (PPE)}
\label{algo:pca}
\begin{algorithmic}[1]
\State \textbf{Input:} Historical updates matrix \(\boldsymbol{V} = [\boldsymbol{v}_1, \dots, \boldsymbol{v}_l] \), number of principal components \(k\), domain scaler \(\alpha\), bulk scaler \(\gamma\)

\State \( \boldsymbol{U}_k  \gets \mathrm{PCA}(\boldsymbol{V}, k)\)
\State $\mathcal{P}_{\alpha,\gamma} \gets \alpha\,\boldsymbol{U}_k\, \boldsymbol{U}_k^\top + \gamma\,(\boldsymbol{I}-\boldsymbol{U}_k\, \boldsymbol{U}_k^\top)$
\State \textbf{Return:} \(\mathcal{P}_{\alpha,\gamma}\)
\end{algorithmic}
\end{algorithm}

To validate the efficacy of PPE, we train a Transformer on SST2 ($k=2$ for this binary classification task) following the setup in Figure~\ref{fig:minhak}, and compare the two-dimensional dominant subspaces estimated by PPE ($U=(u_1,u_2)$) and LPE ($U'=(u'_1,u'_2)$) in Figure~\ref{fig:pca1}. We compute the squared projection norms $\|U^\top u'_i\|^2$ for $i=1,2$ at multiple checkpoints. As shown in Figure~\ref{fig:pca1}, the leading eigenvector attains a projection norm virtually equal to 1, and the second eigenvector remains close to 1 on average. Moreover, Figure~\ref{fig:pca2} demonstrates that BSFA with either estimator yields nearly identical training trajectories and final accuracies. Table~\ref{tab:wall_trans} reports the average per-update time on an RTX4090, showing a 99.84\% speed-up with PPE over LPE. Additional runtime comparisons are provided in Appendix~\ref{appen:time_compare}.

\begin{table}[htbp]
  \centering
  \begin{tabular}{@{}lc@{}}
    \toprule
    Estimator & Time \\
    \midrule
    LPE       & 10.28s   \\
    PPE       & 0.12s  \textbf{($\bm{\downarrow~99.84\%}$)}   \\
    \bottomrule
  \end{tabular}
  \caption{Wall-clock time on 1 RTX4090.}
  \label{tab:wall_trans}
\end{table}

\subsection{Block-wise Strategy for Enhancing BSFA Scalability}
\label{sec:block-wise}

In the preceding sections, we introduced the BSFA framework and presented our PCA estimator, which delivers substantial acceleration in approximating the dominant subspace and thus renders BSFA far more practical. However, directly applying this PCA estimator to large-scale models' entire, high-dimensional parameter vector presents a significant performance bottleneck: the essential SVD operation becomes extremely memory-intensive and computationally slow when processing such massive, monolithic parameter vectors.

\paragraph{\textbf{Leveraging Block‐Diagonal Hessian Structure.}}
To overcome the scalability challenges posed by extremely high‐dimensional parameter vectors, we leverage the intuition that the loss Hessian in deep neural networks is approximately block‐diagonal. Extensive empirical studies have demonstrated that, across architectures—including modern transformers—the Hessian naturally decomposes into independent blocks, each with its eigenvalue spectrum \citep{collobert2004large, roux2007topmoumoute, martens2015optimizing, zhang2024why}.

Motivated by this structure, we introduce the \textbf{Block‐wise PCA‐based Projector Estimator (BPPE)}, detailed in Algorithm~\ref{algo:block_pca}. BPPE constructs a block-diagonal projector by partitioning the full parameter vector into \(B\) disjoint subvectors (typically via PyTorch's default block‐wise segmentation). For each block, our PCA estimator is independently applied to its historical parameter updates to extract the local dominant subspace and form a block-specific projector. In the context of large language models, BPPE strategically excludes the input and output (embedding) layers from these subspace operations—applying only the base optimizer to these large, sparse components—a practice consistent with optimizer designs like Adam‐Mini \citep{adam_mini_2024} and Muon \citep{liu2025muonscalablellmtraining, jordan2024muon}.

\begin{figure*}[t]
  \centering
  \begin{subfigure}[b]{0.32\textwidth}
    \centering
    \includegraphics[width=\linewidth]{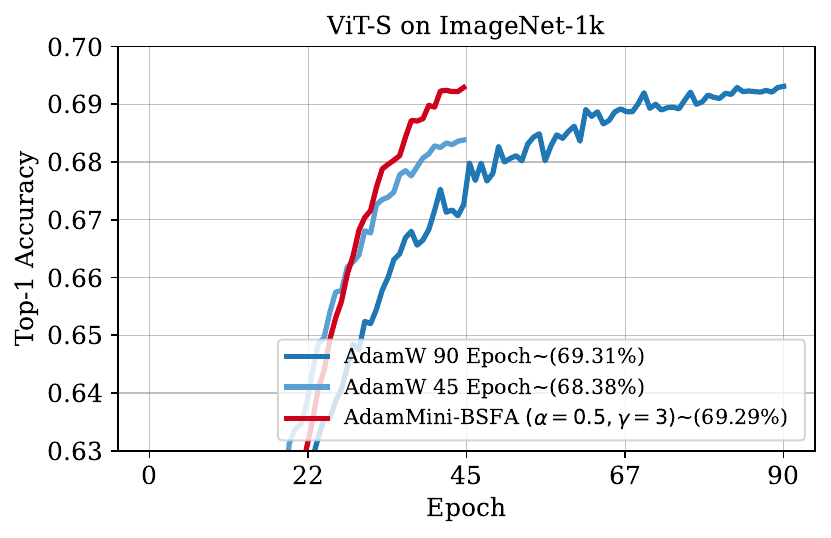}
    \caption{ViT-Small on Imagenet-1k}
    \label{fig:main_results_a}
  \end{subfigure}\hspace{0.01\textwidth}%
  \begin{subfigure}[b]{0.32\textwidth}
    \centering
    \includegraphics[width=\linewidth]{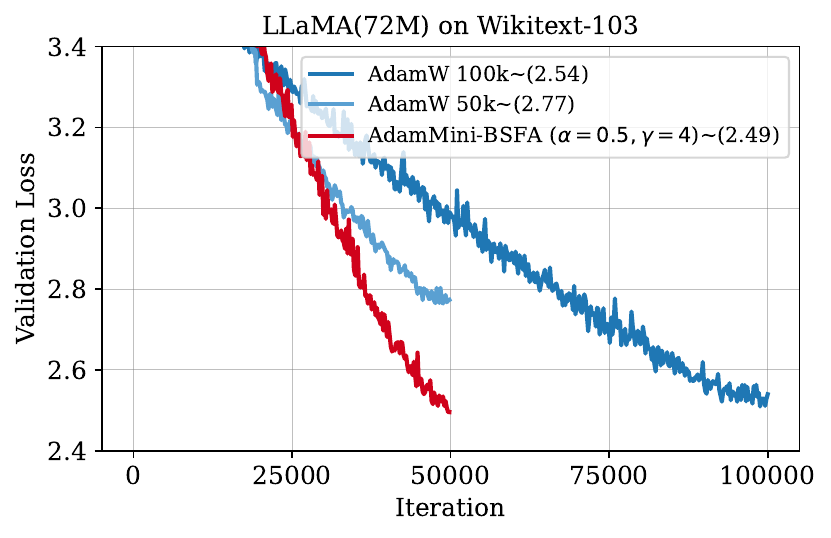}
    \caption{LLaMA(72M) on WikiText-103}
    \label{fig:main_results_b}
  \end{subfigure}\hspace{0.01\textwidth}%
  \begin{subfigure}[b]{0.32\textwidth}
    \centering
    \includegraphics[width=\linewidth]{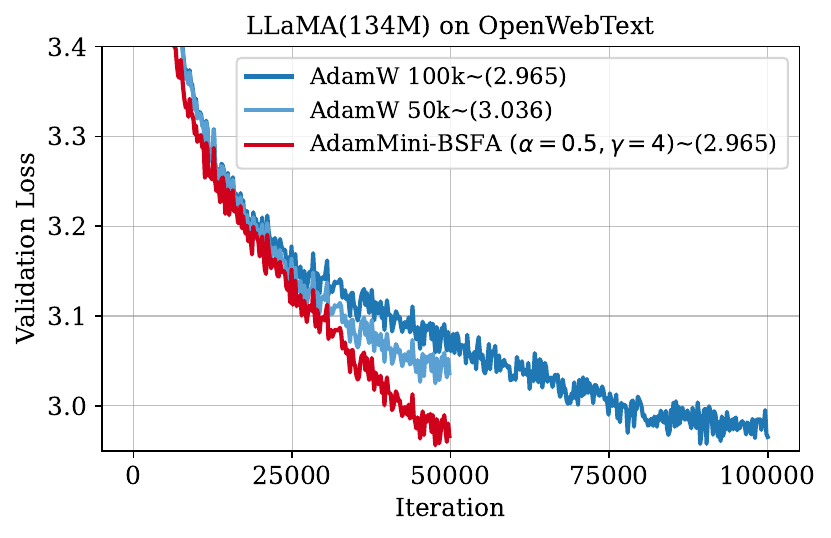}
    \caption{LLaMA(134M) OpenWebText}
    \label{fig:main_results_c}
  \end{subfigure}
  \caption{AdamMini--BSFA outperform AdamW baseline in ViT and LLaMA training.}
  \label{fig:main_results}
\end{figure*}

\begin{algorithm}[htbp]
      \caption{Block‐wise PCA-based Projector Estimator (\textbf{BPPE})}
      \label{algo:block_pca}
      \begin{algorithmic}[1]
        \State \textbf{Input:}
          Historical updates matrix $\boldsymbol{V} = [\boldsymbol{v}_1, \dots, \boldsymbol{v}_l]\in\mathbb R^{p\times l}$, 
          parameter blocks $B$, 
          principal components $k$, 
          domspace scaler $\alpha$, 
          bulkspace scaler $\gamma$
        \State Partition the parameter indices into $\{I_b\}_{b=1}^B$ via PyTorch default partition
        \For{$b$ in $B$}
          \If{LLM \& $b\in\{\text{embedding},\text{output}\}$}
            %\Comment skip BSFA
            \State \textbf{continue}
          \EndIf
          \State $P^{(b)}_{\alpha,\gamma} \gets \textbf{PPE}\bigl(\boldsymbol{V} [I_b,:],\,k,\alpha,\gamma\bigr)$
        \EndFor
        \State $\displaystyle\mathcal P_{\alpha,\gamma}\gets\mathrm{blockdiag}\bigl(P^{(1)}_{\alpha,\gamma},\dots,P^{(B)}_{\alpha,\gamma}\bigr)$
        \State \textbf{Return:} $\mathcal P_{\alpha,\gamma}$
      \end{algorithmic}
    \end{algorithm}

\begin{figure}[H]
  \includegraphics[width=\columnwidth]{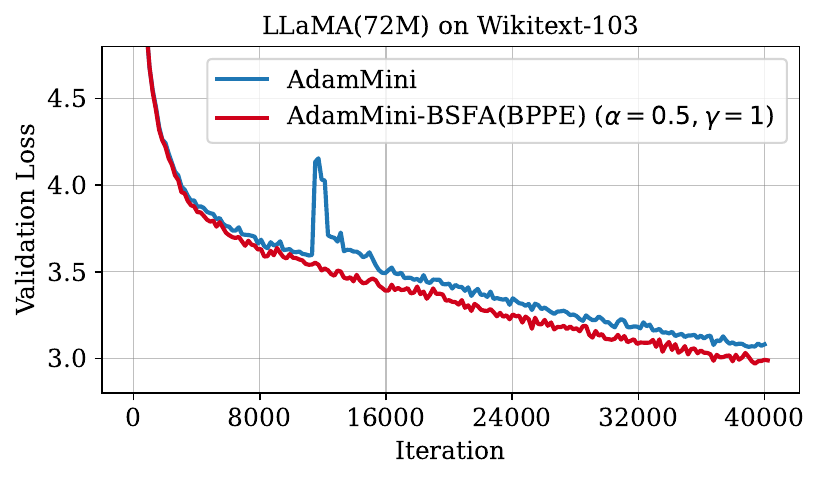}
    \caption{Testing BPPE on LLaMA.}
    \label{fig:bppe}
\end{figure}

This block-wise application via BPPE inherently reduces the memory footprint, improves cache efficiency, and enhances numerical stability and parallelism compared to a dense projector. These advantages collectively accelerate the overall projector estimation and critically improve BSFA's scalability. We evaluate BPPE in Figure~\ref{fig:bppe} by training LLaMA (72M) \citep{touvron2023llamaopenefficientfoundation} on WikiText-103 (detailed settings in Appendix~\ref{appn:part2}). Integrated into AdamMini-BSFA with dominant subspace updates moderated by $\alpha=0.5$, BPPE effectively mitigates training loss spikes and enhances stability, demonstrating its reliability in leveraging block-specific dominant subspace information for improved training dynamics.

\section{Experiments}
\label{sec:part3}

\subsection{Experiment settings}

For language tasks, we examine the performance of BSFA in two experiments: LLaMA-72M on WikiText-103 and LLaMA-134M on OpenWebText \citep{Gokaslan2019OpenWeb}. For vision tasks, we choose  ViT-Small \citep{dosovitskiy2020image} on ImageNet-1k \citep{deng2009imagenet}.

For all experiments, we use the default AdamW as the baseline optimizer, we follow the training protocols of nanoGPT and LLaMA for language tasks, with $\beta_1=0.9$, $\beta_2=0.95$, and weight decay $\lambda=0.1$. The learning rate is linearly warmed to \texttt{lr\_max} and decayed via a cosine scheduler. For each task, we tune \texttt{lr\_max} to optimize AdamW; further details are provided in Appendix~\ref{appn:part3}.

\paragraph{\textbf{BSFA Implementation}} 
We integrate BSFA into AdamMini under the same settings as above.
AdamMini is chosen primarily for two reasons: its ~2x memory reduction compared to AdamW accommodates BPPE's overhead, and its minimal block-wise second-moment statistics ensure a distinct operational mechanism that does not overlap with BSFA’s subspace adjustments, facilitating a clean integration. Note that AdamW’s baseline performance is comparable to or slightly better than AdamMini’s \citep{adam_mini_2024}. This context allows our setup with AdamMini to effectively showcase BSFA's acceleration capabilities. Moreover, because we use BPPE as the estimator, our method's average per-step time is similar to AdamW; Table~\ref{tab:bsfa_time} shows the time comparison when the update interval is $T=10$.

\begin{table}[htbp]
  \centering

  % Reset column separation to LaTeX default
  \setlength{\tabcolsep}{6pt} % Standard LaTeX default

  % Reset row height to LaTeX default
  \renewcommand{\arraystretch}{1.0} % Standard LaTeX default

  % Using @{} to remove default padding at the ends of the table (good practice with booktabs)
  \begin{tabular}{@{}lcc@{}}
    \toprule
    Experiment  & AdamW (/step) & BSFA (/step) \\
    \midrule
    ViT-S           & 0.868s & 1.02s \\
    LLaMA(72M)    & 0.386s & 0.422s \\
    LLaMA(134M)   & 1.86s  & 2.12s \\
    \bottomrule
  \end{tabular}

  % Standard caption spacing (assuming caption package is used)
  \captionsetup{skip=10pt} % This is a common default skip value for the caption package

  \caption{Comparison of average per-step wall-clock time (in seconds) on four RTX4090 GPUs. "BSFA" indicates the AdamMini-BSFA.}
  \label{tab:bsfa_time}
\end{table}

\subsection{Main Results}

In Figure \ref{fig:main_results}, we train ViT‐S on ImageNet-1k and compare a well‐tuned AdamW with AdamW‐BSFA $(\alpha=0.5,\gamma=3)$, 
and we evaluate LLaMA-72M on WikiText-103 and LLaMA-134M on OpenWebText, comparing well‐tuned AdaW with AdamMini‐BSFA $(\alpha=0.5,\gamma=4)$ and AdamMini‐BSFA $(\alpha=0.5,\gamma=2)$. Across all settings, for the same total number of training steps, BSFA consistently achieves lower training loss (and higher top-1 accuracy on ImageNet-1k) than vanilla AdamW or AdamMini. Moreover, BSFA reaches the performance of vanilla AdamW or AdamMini trained for twice as many steps, corresponding to a 2× acceleration.

\subsection{Ablation Study}

In the above experiments, BSFA with BPPE was applied across all major architectural blocks, categorized as \textbf{Norm}, \textbf{Attention}, and \textbf{MLP} blocks. Here, we conduct experiments to investigate the contribution of applying BPPE to different blocks towards acceleration and training stability. All ablations adhere to the BSFA configuration from Figure~\ref{fig:main_results_b}, with $(\alpha=0.5, \gamma=4)$. We apply BSFA solely to the selected blocks, while the remaining blocks continue with vanilla AdamW. As shown in Figure \ref{fig:ablation}, we derive the following two conclusions:

Result is shown in Figure \ref{fig:ablation}. We derive the following two conclusions:

\begin{itemize}[leftmargin=10pt]
    \item Norm layers play a crucial role in maintaining training stability. In Figure \ref{fig:ablation} (Top), applying BSFA to the norm layers does not accelerate training. However, compared to vanilla AdamW, no significant loss spike is observed. When BSFA is not applied to the norm layers, training diverges.
    \item Both Attention and MLP blocks contribute to training acceleration. In Figure \ref{fig:ablation}(Bottom), we compare the acceleration effects of applying BSFA to (Attention, Norm) and (MLP, Norm) blocks (with Norm layers ensuring training stability). Both configurations accelerate training relative to vanilla AdamW, achieving similar acceleration slightly inferior to block-wise BSFA applied to all blocks (Attention, MLP, Norm), indicating that both Attention and MLP blocks contribute comparably to accelerating the model.
\end{itemize}

{ % Start a group to localize the change to \textfloatsep
\setlength{\textfloatsep}{3pt plus 2pt minus 2pt} % Reduce space above/below float

\begin{figure}[H] % Or [!t] for more urgent placement
  \centering
  \includegraphics[width=1\columnwidth]{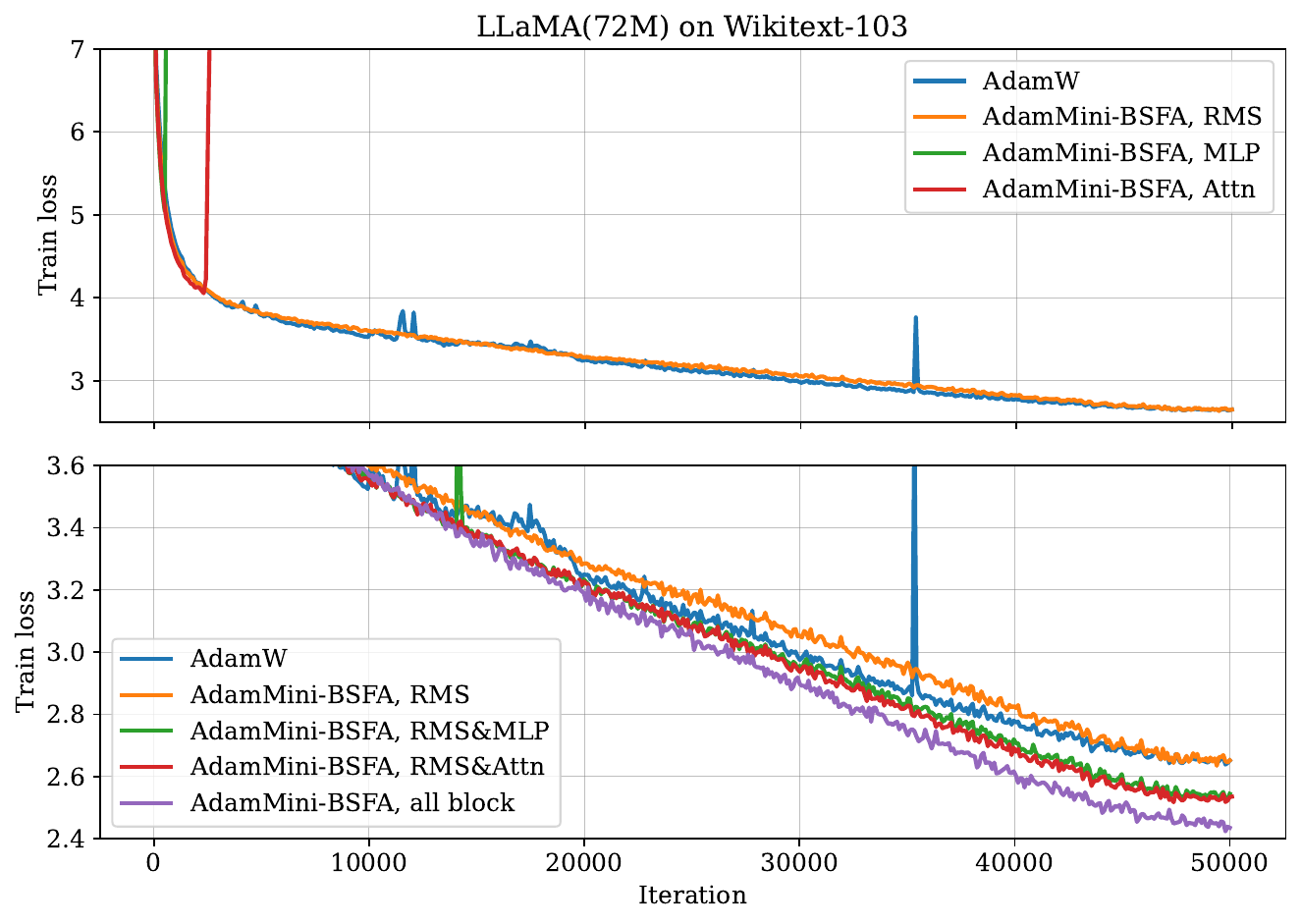}
  % No subcaption, no label for this image if not needed

  % Adjust space above main caption and ensure minimal space below it.
  % \captionsetup from the 'caption' package is good for this.
  % Ensure you have \usepackage{caption} in your preamble.
  \captionsetup{aboveskip=4pt, belowskip=0pt} % Small space above caption, minimal below. Adjust as needed.

  \caption{ (Top) Diverge if not applying BSFA to Norm blocks. (Bottom)Apply BSFA  to either attention or MLP blocks alongside norm layers.}
  \label{fig:ablation} % Main figure label
\end{figure}

} 

\subsection{Memory Footprint and Efficiency of BSFA}

\begin{table*}[t]
  \centering
  \begin{tabular}{@{}lcccccc@{}}
    \toprule
    Optimizer & micro\_BS & Memory & max\_iter & avg\_per\_step\_time & total\_time & final\_val\_loss \\
    \midrule
    AdamW           & 30 & 10.2GB & 100k & 386ms & \textasciitilde10.72h & 2.54 \\
    AdamW           & 60 & 17.6GB & 100k & 335ms & \textasciitilde9.32h  & 2.54 \\
    \textbf{BSFA} & \textbf{30} & \textbf{18.2GB} & \textbf{50k} & \textbf{422ms} & \textbf{\textasciitilde5.86h} & \textbf{2.49} \\
    \bottomrule
  \end{tabular}
  \caption{Comparison with a memory budget similar to AdamW with a doubled micro-batch size. "Memory" indicates peak memory usage per GPU}
  \label{tab:optimizer_time_compare}
\end{table*}

BSFA explicitly caches the top-$k$ principal directions to build its projector, incurring an additional memory overhead equivalent to storing $k$ full gradients. Because memory is a primary bottleneck in training, we examine: (i) whether the additional memory indeed translates into faster training, and (ii) can low-bit quantization reduce BSFA’s memory usage without degrading performance?

\paragraph{\textbf{The Memory-Performance Trade-off of BSFA}}

To assess whether the additional memory consumed by BSFA is justified by its performance benefits, We test whether BSFA’s extra memory is better spent on BSFA itself or on larger micro-batches. On LLaMA-72M, BSFA adds ~8 GB/GPU (storing $30 \times$ gradient-sized dominant directions), while increasing the micro-batch from 30 to 60 costs ~7.4 GB/GPU in activation memory. We compare their time overheads in Table \ref{tab:optimizer_time_compare}. This result demonstrates that in pre-training LLaMA-72M, allocating extra memory to BSFA is a more effective strategy for acceleration. A key insight is that while additional memory can be used to enlarge the micro-batch size, GPU utilization is already near saturation. Consequently, any performance gains from further increasing the micro-batch size become marginal.

\paragraph{\textbf{Potential for Memory Reduction: 4-bit BSFA}}

Although the memory-performance trade-off ostensibly favors BSFA, its auxiliary memory requirements may become prohibitive as model sizes continue to scale. To explore reducing BSFA’s memory footprint, we evaluate whether quantizing BSFA can lower memory usage without degrading performance. Following the exact protocol of Figure \ref{fig:main_results_b}, we implement a 4-bit variant that linearly quantizes BSFA’s historical gradients and projection matrix. Specifically, we apply 4-bit linear quantization to the historical gradients and the projection matrix required by BSFA. The per-step updates and the dominant directions are quantized using a \texttt{group\_size} of 64 and stored compactly via nibble-packing. For each group, the scale and zero-point are stored in the FP8 (E4M3) format. These values are only dequantized during the parameter update step, which significantly conserves memory. We perform a comparative study following the exact experimental setup of Figure \ref{fig:main_results_b}. The table below presents the validation loss for each optimizer at various training steps.

\begin{table}[htbp]
  \centering

  % Reset column separation to LaTeX default
  \setlength{\tabcolsep}{6pt} % Standard LaTeX default

  % Reset row height to LaTeX default
  \renewcommand{\arraystretch}{1.0} % Standard LaTeX default

  % Using @{} to remove default padding at the ends of the table (good practice with booktabs)
  \begin{tabular}{@{}lcc@{}}
    \toprule
    Optimizer  & Loss@50k & Peak Memory \\
    \midrule
    AdamW           & 2.77 & 10.2GB \\
    BSFA    & 2.49 & 18.2GB \\
    \textbf{4bit-BSFA}   & \textbf{2.54} & \textbf{11.5GB} \\
    \bottomrule
  \end{tabular}

  % Standard caption spacing (assuming caption package is used)
  \captionsetup{skip=10pt} % This is a common default skip value for the caption package

  \caption{Comparison of peak memory usage and final loss on four RTX4090 GPUs. "BSFA" indicates the AdamMini with BSFA, "4bit-BSFA" indicates the AdamMini with 4bit-BSFA, "Peak Memory" indicates peak memory usage per GPU.}
  \label{tab:bsfa_memory_foot}
\end{table}

The results below show that 4-bit BSFA achieves performance comparable to the full-precision version while reducing the additional memory overhead from $8.0$ GB to just $1.3$ GB, approximately $1/6$ of its initial value. This result shows that quantizing BSFA to 4 bits is an effective memory-reduction strategy that preserves training quality, providing a practical path to further improve BSFA’s memory efficiency. We emphasize that the core contribution of this work is the insightful acceleration strategy derived from our analysis of the loss landscape, leaving the development of more practical, lightweight variants for future work. 
\section{Conclusion}
\label{sec:conclusion}

In this paper, we introduced the Bulk-Space-Filtration-Accelerator (BSFA), a novel framework that exploits the distinct roles of dominant and bulk subspaces in neural network training. By differentially scaling updates within these subspaces—moderating the dominant for stability and amplifying the bulk for speed—and leveraging an efficient PCA-based estimator with a block-wise strategy for scalability, BSFA achieves significant training acceleration, notably an approximate 2$\times$ speedup for large Transformer models like LLaMA and ViT compared to AdamW.
%\newpage

\section*{Limitations}

Although BSFA demonstrates promising acceleration capabilities, it has several areas for potential improvement. First, the PCA estimator BPPE demands significant GPU memory, and reducing its memory requirements remains challenging. Furthermore, estimating fewer dominant directions to save memory could compromise the accuracy of the Projector, consequently diminishing the acceleration benefits. Second, our current approach assigns an equal number of dominant directions to each block, which may not optimally accommodate the heterogeneous properties of different blocks. Future work could focus on mitigating these memory constraints and developing more adaptive strategies for block-specific control. In this work, we use up to 4 RTX4090 to train for less than 500 hours, including hyperparameter tuning. We wish this experience could boost the understanding of efficient training in the community.

\section*{Ethical Considerations}

This paper seeks to promote the development of deep learning by focusing on understanding and improving the training processes of neural networks, with particular emphasis on large language models (LLMs). While our work has the potential for broad societal impact, we do not, at present, identify any specific societal implications that warrant special attention.

\section*{Acknowledgements}

This work was supported by the Strategic Priority Research Program of the Chinese Academy of Sciences (Grant No. XDB0680101), CAS Project for Young Scientists in Basic Research under Grant No. YSBR-034.

\bibliography{main}

\begin{thebibliography}{56}
\providecommand{\natexlab}[1]{#1}

\bibitem[{tie(2012)}]{tieleman2012lecture}
 2012.
\newblock Lecture 6.5-rmsprop: Divide the gradient by a running average of its recent magnitude.
\newblock \emph{COURSERA: Neural networks for machine learning}, 4(2):26.

\bibitem[{Achiam et~al.(2023)Achiam, Adler, Agarwal, Ahmad, Akkaya, Aleman, Almeida, Altenschmidt, Altman, Anadkat et~al.}]{achiam2023gpt}
Josh Achiam, Steven Adler, Sandhini Agarwal, Lama Ahmad, Ilge Akkaya, Florencia~Leoni Aleman, Diogo Almeida, Janko Altenschmidt, Sam Altman, Shyamal Anadkat, and 1 others. 2023.
\newblock Gpt-4 technical report.
\newblock \emph{arXiv preprint arXiv:2303.08774}.

\bibitem[{Beyer et~al.(2022)Beyer, Zhai, and Kolesnikov}]{beyer2022betterplainvitbaselines}
Lucas Beyer, Xiaohua Zhai, and Alexander Kolesnikov. 2022.
\newblock \href {https://arxiv.org/abs/2205.01580} {Better plain vit baselines for imagenet-1k}.
\newblock \emph{Preprint}, arXiv:2205.01580.

\bibitem[{Chen et~al.(2023)Chen, Liang, Huang, Real, Wang, Liu, Pham, Dong, Luong, Hsieh, Lu, and Le}]{chen2023symbolic}
Xiangning Chen, Chen Liang, Da~Huang, Esteban Real, Kaiyuan Wang, Yao Liu, Hieu Pham, Xuanyi Dong, Thang Luong, Cho-Jui Hsieh, Yifeng Lu, and Quoc~V. Le. 2023.
\newblock Symbolic discovery of optimization algorithms.
\newblock \emph{arXiv preprint arXiv:2302.06675}.

\bibitem[{Cohen et~al.(2021)Cohen, Kaur, Jiang, Kolter, and Talwalkar}]{cohen2021gradient}
Jeremy Cohen, Simran Kaur, Yiding Jiang, Zico Kolter, and Ameet Talwalkar. 2021.
\newblock Gradient descent on neural networks typically occurs at the edge of stability.
\newblock \emph{arXiv preprint arXiv:2103.00065}.

\bibitem[{Collobert(2004)}]{collobert2004large}
Ronan Collobert. 2004.
\newblock Large scale machine learning.

\bibitem[{Damian et~al.(2022)Damian, Nichani, and Lee}]{damian2022self}
Alex Damian, Eshaan Nichani, and Jason~D Lee. 2022.
\newblock Self-stabilization: The implicit bias of gradient descent at the edge of stability.
\newblock \emph{arXiv preprint arXiv:2209.15594}.

\bibitem[{Deng et~al.(2009)Deng, Dong, Socher, Li, Li, and Fei-Fei}]{deng2009imagenet}
Jia Deng, Wei Dong, Richard Socher, Li-Jia Li, Kai Li, and Li~Fei-Fei. 2009.
\newblock Imagenet: A large-scale hierarchical image database.
\newblock In \emph{2009 IEEE conference on computer vision and pattern recognition}, pages 248--255. Ieee.

\bibitem[{Devlin et~al.(2019)Devlin, Chang, Lee, and Toutanova}]{devlin2018bert}
Jacob Devlin, Ming-Wei Chang, Kenton Lee, and Kristina Toutanova. 2019.
\newblock Bert: Pre-training of deep bidirectional transformers for language understanding.
\newblock In \emph{Proceedings of the 2019 Conference of the North American Chapter of the Association for Computational Linguistics: Human Language Technologies, Volume 1 (Long and Short Papers)}, pages 4171--4186.

\bibitem[{Dosovitskiy et~al.(2020)Dosovitskiy, Beyer, Kolesnikov, Weissenborn, Zhai, Unterthiner, Dehghani, Minderer, Heigold, Gelly et~al.}]{dosovitskiy2020image}
Alexey Dosovitskiy, Lucas Beyer, Alexander Kolesnikov, Dirk Weissenborn, Xiaohua Zhai, Thomas Unterthiner, Mostafa Dehghani, Matthias Minderer, Georg Heigold, Sylvain Gelly, and 1 others. 2020.
\newblock An image is worth 16x16 words: Transformers for image recognition at scale.
\newblock \emph{arXiv preprint arXiv:2010.11929}.

\bibitem[{Duchi et~al.(2011)Duchi, Hazan, and Singer}]{duchi2011adaptive}
John Duchi, Elad Hazan, and Yoram Singer. 2011.
\newblock Adaptive subgradient methods for online learning and stochastic optimization.
\newblock \emph{Journal of Machine Learning Research}, 12:2121--2159.

\bibitem[{Esteva et~al.(2019)}]{esteva2019guide}
Andre Esteva and 1 others. 2019.
\newblock A guide to deep learning in healthcare.
\newblock \emph{Nature Medicine}, 25:24--29.

\bibitem[{Freund and Schapire(1996)}]{freund1996experiments}
Yoav Freund and Robert~E. Schapire. 1996.
\newblock Experiments with a new boosting algorithm.
\newblock In \emph{Proceedings of the 13th International Conference on Machine Learning}, pages 148--156.

\bibitem[{Ghorbani et~al.(2019)Ghorbani, Krishnan, and Xiao}]{pmlr-v97-ghorbani19b}
Behrooz Ghorbani, Shankar Krishnan, and Ying Xiao. 2019.
\newblock \href {https://proceedings.mlr.press/v97/ghorbani19b.html} {An investigation into neural net optimization via hessian eigenvalue density}.
\newblock In \emph{Proceedings of the 36th International Conference on Machine Learning}, volume~97 of \emph{Proceedings of Machine Learning Research}, pages 2232--2241. PMLR.

\bibitem[{Gokaslan and Cohen(2019)}]{Gokaslan2019OpenWeb}
Aaron Gokaslan and Vanya Cohen. 2019.
\newblock Openwebtext corpus.
\newblock \url{http://Skylion007.github.io/OpenWebTextCorpus}.

\bibitem[{Gur-Ari et~al.(2018)Gur-Ari, Roberts, and Dyer}]{gur2018gradient}
Guy Gur-Ari, Daniel~A Roberts, and Ethan Dyer. 2018.
\newblock Gradient descent happens in a tiny subspace.
\newblock \emph{arXiv preprint arXiv:1812.04754}.

\bibitem[{He et~al.(2016)He, Zhang, Ren, and Sun}]{he2016deep}
Kaiming He, Xiangyu Zhang, Shaoqing Ren, and Jian Sun. 2016.
\newblock Deep residual learning for image recognition.
\newblock In \emph{Proceedings of the IEEE conference on computer vision and pattern recognition}, pages 770--778.

\bibitem[{Hochreiter and Schmidhuber(1997)}]{hochreiter1997flat}
Sepp Hochreiter and J{\"u}rgen Schmidhuber. 1997.
\newblock Flat minima.
\newblock \emph{Neural computation}, 9(1):1--42.

\bibitem[{Hoffmann et~al.(2022)Hoffmann, Borgeaud, Mensch, Buchatskaya, Cai, Rutherford, de~Las~Casas, Hendricks, Welbl, Clark et~al.}]{hoffmann2022training}
Jordan Hoffmann, Sebastian Borgeaud, Arthur Mensch, Elena Buchatskaya, Trevor Cai, Eliza Rutherford, Diego de~Las~Casas, Lisa~Anne Hendricks, Johannes Welbl, Aidan Clark, and 1 others. 2022.
\newblock Training compute-optimal large language models.
\newblock \emph{arXiv preprint arXiv:2203.15556}.

\bibitem[{Hu et~al.(2024)Hu, Tu, Han, He, Cui, Long, Zheng, Fang, Huang, Zhao et~al.}]{hu2024minicpm}
Shengding Hu, Yuge Tu, Xu~Han, Chaoqun He, Ganqu Cui, Xiang Long, Zhi Zheng, Yewei Fang, Yuxiang Huang, Weilin Zhao, and 1 others. 2024.
\newblock Minicpm: Unveiling the potential of small language models with warmup-stable-decay learning rate scheduler.
\newblock \emph{arXiv preprint arXiv:2404.06395}.

\bibitem[{Jastrzębski et~al.(2018)Jastrzębski, Kenton, Arpit, Ballas, Fischer, Bengio, and Storkey}]{jastrzebski2018width}
Stanisław Jastrzębski, Zachary Kenton, Devansh Arpit, Nicolas Ballas, Asja Fischer, Yoshua Bengio, and Amos Storkey. 2018.
\newblock Width of minima reached by stochastic gradient descent is influenced by learning rate to batch size ratio.
\newblock In \emph{Artificial Neural Networks and Machine Learning - ICANN 2018}, pages 392--402. Springer International Publishing.

\bibitem[{Jastrzębski et~al.(2020)Jastrzębski, Szymczak, Fort, Arpit, Tabor, Cho, and Geras}]{jastrzebski2020break}
Stanisław Jastrzębski, Maciej Szymczak, Stanislav Fort, Devansh Arpit, Jacek Tabor, Kyunghyun Cho, and Krzysztof Geras. 2020.
\newblock \href {https://arxiv.org/abs/2002.09572} {The break-even point on optimization trajectories of deep neural networks}.
\newblock In \emph{International Conference on Learning Representations}.

\bibitem[{Jordan et~al.(2024)Jordan, Jin, Boza, Jiacheng, Cesista, Newhouse, and Bernstein}]{jordan2024muon}
Keller Jordan, Yuchen Jin, Vlado Boza, You Jiacheng, Franz Cesista, Laker Newhouse, and Jeremy Bernstein. 2024.
\newblock \href {https://kellerjordan.github.io/posts/muon/} {Muon: An optimizer for hidden layers in neural networks}.

\bibitem[{Kaplan et~al.(2020)Kaplan, McCandlish, Henighan, Brown, Chess, Child, Gray, Radford, Wu, and Amodei}]{kaplan2020scaling}
Jared Kaplan, Sam McCandlish, Tom Henighan, Tom~B. Brown, Benjamin Chess, Rewon Child, Scott Gray, Alec Radford, Jeffrey Wu, and Dario Amodei. 2020.
\newblock Scaling laws for neural language models.
\newblock \emph{arXiv preprint arXiv:2001.08361}.

\bibitem[{Kingma and Ba(2014)}]{kingma2014adam}
Diederik~P Kingma and Jimmy Ba. 2014.
\newblock Adam: A method for stochastic optimization.
\newblock \emph{arXiv preprint arXiv:1412.6980}.

\bibitem[{Li et~al.(2018)Li, Xu, Taylor, Studer, and Goldstein}]{li2018visualizing}
Hao Li, Zheng Xu, Gavin Taylor, Christoph Studer, and Tom Goldstein. 2018.
\newblock \href {https://proceedings.neurips.cc/paper/2018/hash/a41b3bb3e6b050b6c9067c67f663b915-Abstract.html} {Visualizing the loss landscape of neural nets}.
\newblock In \emph{Advances in Neural Information Processing Systems (NeurIPS)}.

\bibitem[{Liu et~al.(2023)Liu, Li, Hall, Liang, and Ma}]{liu2023sophia}
Hong Liu, Zhiyuan Li, David Hall, Percy Liang, and Tengyu Ma. 2023.
\newblock Sophia: A scalable stochastic second-order optimizer for language model pre-training.
\newblock \emph{arXiv preprint arXiv:2305.14342}.

\bibitem[{Liu et~al.(2025)Liu, Su, Yao, Jiang, Lai, Du, Qin, Xu, Lu, Yan, Chen, Zheng, Liu, Liu, Yin, He, Zhu, Wang, Wang, Dong, Zhang, Kang, Zhang, Xu, Zhang, Wu, Zhou, and Yang}]{liu2025muonscalablellmtraining}
Jingyuan Liu, Jianlin Su, Xingcheng Yao, Zhejun Jiang, Guokun Lai, Yulun Du, Yidao Qin, Weixin Xu, Enzhe Lu, Junjie Yan, Yanru Chen, Huabin Zheng, Yibo Liu, Shaowei Liu, Bohong Yin, Weiran He, Han Zhu, Yuzhi Wang, Jianzhou Wang, and 9 others. 2025.
\newblock \href {https://arxiv.org/abs/2502.16982} {Muon is scalable for llm training}.
\newblock \emph{Preprint}, arXiv:2502.16982.

\bibitem[{Liu et~al.(2020)}]{liu2020deep}
Yu~Liu and 1 others. 2020.
\newblock A deep learning system for differential diagnosis of skin diseases.
\newblock \emph{Nature Medicine}, 26:900--908.

\bibitem[{Liu et~al.(2021)Liu, Lin, Cao, Hu, Wei, Zhang, Lin, and Guo}]{liu2021swin}
Ze~Liu, Yutong Lin, Yue Cao, Han Hu, Yixuan Wei, Zheng Zhang, Stephen Lin, and Baining Guo. 2021.
\newblock Swin transformer: Hierarchical vision transformer using shifted windows.
\newblock In \emph{Proceedings of the IEEE/CVF international conference on computer vision}, pages 10012--10022.

\bibitem[{Martens and Grosse(2015)}]{martens2015optimizing}
James Martens and Roger Grosse. 2015.
\newblock Optimizing neural networks with kronecker-factored approximate curvature.
\newblock In \emph{International conference on machine learning}, pages 2408--2417. PMLR.

\bibitem[{Miotto et~al.(2018)Miotto, Wang, Wang, Jiang, and Dudley}]{miotto2018deep}
Riccardo Miotto, Fei Wang, Shuang Wang, Xiaoqian Jiang, and Joel~T Dudley. 2018.
\newblock Deep learning for healthcare: review, opportunities and challenges.
\newblock \emph{Briefings in Bioinformatics}, 19(6):1236--1246.

\bibitem[{Molybog et~al.(2023)Molybog, Albert, Chen, DeVito, Esiobu, Goyal, Koura, Narang, Poulton, Silva, Tang, Xu, Zhang, Kambadur, Roller, and Zhang}]{molybog2023theory}
Igor Molybog, Peter Albert, Moya Chen, Zachary DeVito, David Esiobu, Naman Goyal, Punit~Singh Koura, Sharan Narang, Andrew Poulton, Ruan Silva, Binh Tang, Puxin Xu, Yuchen Zhang, Melanie Kambadur, Stephen Roller, and Susan Zhang. 2023.
\newblock A theory on adam instability in large-scale machine learning.
\newblock \emph{arXiv preprint arXiv:2304.09871}.

\bibitem[{Papyan(2018)}]{Papyan2018TheFS}
Vardan Papyan. 2018.
\newblock \href {https://api.semanticscholar.org/CorpusID:173990834} {The full spectrum of deepnet hessians at scale: Dynamics with sgd training and sample size.}
\newblock \emph{arXiv: Learning}.

\bibitem[{Pearlmutter(1994)}]{Pearlmutter1994}
Barak Pearlmutter. 1994.
\newblock \href {https://doi.org/10.1162/neco.1994.6.1.147} {Fast exact multiplication by the hessian}.
\newblock \emph{Neural Computation}, 6:147--160.

\bibitem[{Platt(1998)}]{platt1998sequential}
John~C. Platt. 1998.
\newblock Sequential minimal optimization: A fast algorithm for training support vector machines.
\newblock In \emph{Advances in Neural Information Processing Systems (NIPS)}.

\bibitem[{Polyak(1964)}]{polyak1964some}
Boris~T Polyak. 1964.
\newblock \href {https://doi.org/10.1016/0041-5553(64)90137-5} {Some methods of speeding up the convergence of iteration methods}.
\newblock \emph{USSR Computational Mathematics and Mathematical Physics}, 4:1--17.

\bibitem[{Robbins and Monro(1951)}]{robbins1951stochastic}
Herbert Robbins and Sutton Monro. 1951.
\newblock A stochastic approximation method.
\newblock \emph{Annals of Mathematical Statistics}, 22(3):400--407.

\bibitem[{Roulet et~al.(2024)Roulet, Agarwala, Grill, Swirszcz, Blondel, and Pedregosa}]{roulet2024stepping}
Vincent Roulet, Atish Agarwala, Jean-Bastien Grill, Grzegorz Swirszcz, Mathieu Blondel, and Fabian Pedregosa. 2024.
\newblock Stepping on the edge: Curvature aware learning rate tuners.
\newblock In \emph{Advances in Neural Information Processing Systems (NeurIPS)}.

\bibitem[{Roux et~al.(2007)Roux, Manzagol, and Bengio}]{roux2007topmoumoute}
Nicolas Roux, Pierre-Antoine Manzagol, and Yoshua Bengio. 2007.
\newblock Topmoumoute online natural gradient algorithm.
\newblock \emph{Advances in neural information processing systems}, 20.

\bibitem[{Sagun et~al.(2017)Sagun, Evci, Guney, Dauphin, and Bottou}]{sagun2017empirical}
Levent Sagun, Utku Evci, V~Ugur Guney, Yann Dauphin, and Leon Bottou. 2017.
\newblock Empirical analysis of the hessian of over-parametrized neural networks.
\newblock \emph{arXiv preprint arXiv:1706.04454}.

\bibitem[{Shazeer(2020)}]{shazeer2020glu}
Noam Shazeer. 2020.
\newblock Glu variants improve transformer.
\newblock \emph{arXiv preprint arXiv:2002.05202}.

\bibitem[{Socher et~al.(2013)Socher, Perelygin, Wu, Chuang, Manning, Ng, and Potts}]{socher2013recursive}
Richard Socher, Alex Perelygin, Jean Wu, Jason Chuang, Christopher~D Manning, Andrew~Y Ng, and Christopher Potts. 2013.
\newblock Recursive deep models for semantic compositionality over a sentiment treebank.
\newblock In \emph{Proceedings of the 2013 conference on empirical methods in natural language processing}, pages 1631--1642.

\bibitem[{Song et~al.(2024)Song, Ahn, and Yun}]{song2024does}
Minhak Song, Kwangjun Ahn, and Chulhee Yun. 2024.
\newblock Does sgd really happen in tiny subspaces?
\newblock \emph{arXiv preprint arXiv:2405.16002}.

\bibitem[{Su et~al.(2024)Su, Ahmed, Lu, Pan, Bo, and Liu}]{su2024roformer}
Jianlin Su, Murtadha Ahmed, Yu~Lu, Shengfeng Pan, Wen Bo, and Yunfeng Liu. 2024.
\newblock Roformer: Enhanced transformer with rotary position embedding.
\newblock \emph{Neurocomputing}, 568:127063.

\bibitem[{Touvron et~al.(2023)Touvron, Lavril, Izacard, Martinet, Lachaux, Lacroix, Rozière, Goyal, Hambro, Azhar, Rodriguez, Joulin, Grave, and Lample}]{touvron2023llamaopenefficientfoundation}
Hugo Touvron, Thibaut Lavril, Gautier Izacard, Xavier Martinet, Marie-Anne Lachaux, Timothée Lacroix, Baptiste Rozière, Naman Goyal, Eric Hambro, Faisal Azhar, Aurelien Rodriguez, Armand Joulin, Edouard Grave, and Guillaume Lample. 2023.
\newblock \href {https://arxiv.org/abs/2302.13971} {Llama: Open and efficient foundation language models}.
\newblock \emph{Preprint}, arXiv:2302.13971.

\bibitem[{Vaswani et~al.(2017)Vaswani, Shazeer, Parmar, Uszkoreit, Jones, Gomez, Kaiser et~al.}]{vaswani2017attention}
Ashish Vaswani, Noam Shazeer, Niki Parmar, Jakob Uszkoreit, Llion Jones, Aidan~N. Gomez, Lukasz Kaiser, and 1 others. 2017.
\newblock Attention is all you need.
\newblock In \emph{Advances in Neural Information Processing Systems}, pages 5998--6008.

\bibitem[{Wang et~al.(2025)Wang, Wang, Zhou, Yan, E, and Wu}]{wang2025sharpnessdisparityprincipletransformers}
Jinbo Wang, Mingze Wang, Zhanpeng Zhou, Junchi Yan, Weinan E, and Lei Wu. 2025.
\newblock \href {https://arxiv.org/abs/2502.19002} {The sharpness disparity principle in transformers for accelerating language model pre-training}.
\newblock \emph{Preprint}, arXiv:2502.19002.

\bibitem[{Wang et~al.(2024)Wang, Wang, He, Wang, Huang, Xiong, Li, E, and Wu}]{NEURIPS2024_d712c862}
Mingze Wang, Jinbo Wang, Haotian He, Zilin Wang, Guanhua Huang, Feiyu Xiong, Zhiyu Li, Weinan E, and Lei Wu. 2024.
\newblock \href {https://doi.org/10.52202/079017-3770} {Improving generalization and convergence by enhancing implicit regularization}.
\newblock In \emph{Advances in Neural Information Processing Systems}, volume~37, pages 118701--118744. Curran Associates, Inc.

\bibitem[{Wen et~al.(2024)Wen, Li, Wang, Hall, Liang, and Ma}]{wen2024understanding}
Kaiyue Wen, Zhiyuan Li, Jason Wang, David Hall, Percy Liang, and Tengyu Ma. 2024.
\newblock Understanding warmup-stable-decay learning rates: A river valley loss landscape perspective.
\newblock \emph{arXiv preprint arXiv:2410.05192}.

\bibitem[{Wu et~al.(2018)Wu, Ma, and E}]{wu2018sgd}
Lei Wu, Chao Ma, and Weinan E. 2018.
\newblock How sgd selects the global minima in over-parameterized learning: A dynamical stability perspective.
\newblock In \emph{Advances in Neural Information Processing Systems}, pages 8279--8288.

\bibitem[{Zhang and Sennrich(2019)}]{zhang2019root}
Biao Zhang and Rico Sennrich. 2019.
\newblock Root mean square layer normalization.
\newblock \emph{Advances in Neural Information Processing Systems}, 32.

\bibitem[{Zhang et~al.(2020)Zhang, He, Sra, and Jadbabaie}]{zhang2020gradient}
Jingzhao Zhang, Tianxing He, Suvrit Sra, and Ali Jadbabaie. 2020.
\newblock Why gradient clipping accelerates training: A theoretical justification for adaptivity.
\newblock In \emph{International Conference on Learning Representations (ICLR)}.

\bibitem[{Zhang et~al.(2024{\natexlab{a}})Zhang, Chen, Ding, Li, Sun, and Luo}]{zhang2024why}
Yushun Zhang, Congliang Chen, Tian Ding, Ziniu Li, Ruoyu Sun, and Zhiquan Luo. 2024{\natexlab{a}}.
\newblock Why transformers need adam: A hessian perspective.
\newblock In \emph{Advances in Neural Information Processing Systems}.

\bibitem[{Zhang et~al.(2024{\natexlab{b}})Zhang, Chen, Shi, Sun, and Luo}]{adam_mini_2024}
Yushun Zhang, Congliang Chen, Naichen Shi, Ruoyu Sun, and Zhi-Quan Luo. 2024{\natexlab{b}}.
\newblock Adam-mini: Use fewer learning rates to gain more.
\newblock \emph{arXiv preprint arXiv:2406.16793}.

\bibitem[{Zhu et~al.(2024)Zhu, Liu, Radhakrishnan, and Belkin}]{pmlr-v235-zhu24h}
Libin Zhu, Chaoyue Liu, Adityanarayanan Radhakrishnan, and Mikhail Belkin. 2024.
\newblock \href {https://proceedings.mlr.press/v235/zhu24h.html} {Catapults in {SGD}: spikes in the training loss and their impact on generalization through feature learning}.
\newblock In \emph{Proceedings of the 41st International Conference on Machine Learning}, volume 235 of \emph{Proceedings of Machine Learning Research}, pages 62476--62509. PMLR.

\end{thebibliography}

% 附录部分
\appendix
\section{Experiment Details}
\label{appn:experiment_setting}

\subsection{Experiment Details in Section~\ref{sec:part1}}
\label{appn:par1}
\textbf{Figures~\ref{fig:minhak} \& \ref{fig:demonstra}: Transformer on SST-2.} We conduct illustrative experiments on a two-layer Transformer (hidden dimension $64$, $8$ attention heads) following \citet{damian2022self} and \citet{song2024does}, training on SST2 \citep{socher2013recursive} for binary sentiment classification with cross-entropy loss. All runs use SGD with constant learning rate $\eta=0.03$ and batch size $200$, and stop when training accuracy reaches $0.99$.
Training is performed on a single NVIDIA RTX 4090 GPU.

\subsection{Experiment Details in Section~\ref{sec:part2}}
\label{appn:part2}
\textbf{Figures~\ref{fig:cifar10} \& \ref{fig:cifar100}: CNNs on CIFAR-10/100.} We build on the official PyTorch tutorial, applying random horizontal flips and $32\times32$ crops with $4$-pixel reflection padding, and normalize inputs to mean $0.5$. For SGDM we use initial learning rate $\alpha_0=0.1$, momentum $0.9$, weight decay $5\times10^{-4}$, batch size $1024$, and train for $200$ or $50$ epochs with cosine decay. For BSFA we adopt LPE with $k=10$ (CIFAR-10) or $k=100$ (CIFAR-100), historical updates length $l=k+10$, $T=10$, train for $50$ epochs, and tune $\alpha\in\{1,0.5,0.2\}$, $\gamma\in\{2,4,6\}$ via grid search. We avoid selecting \(\alpha\) values approaching zero, since imprecise projector estimates would unduly attenuate updates in other directions. All experiments run on a single NVIDIA RTX 4090 GPU.

\textbf{Figures~\ref{fig:pca_var} \& \ref{fig:combined}: Transformer on SST-2.} We follow exactly the setup of Section~\ref{appn:par1}.

\textbf{Figure~\ref{fig:bppe}: LLaMA(72M) on Wikitext-103.} We evaluate BSFA on LLaMA \citep{touvron2023llamaopenefficientfoundation}, a decode-only Transformer with RoPE \citep{su2024roformer}, SwiGLU \citep{shazeer2020glu}, and RMSNorm \citep{zhang2019root}, pre-trained on Wikitext-103 ($103$M tokens, $28$K articles). The $72$M-parameter model has $16$ layers, $10$ heads, hidden size $410$, sequence length $150$, batch size $240$. We use AdamMini \citep{adam_mini_2024} with $(\beta_1,\beta_2,\lambda)=(0.9,0.95,0.1)$, learning rate $\eta=5\times10^{-4}$, total steps $40\,000$, $500$-step warm-up to \texttt{lr\_max} then cosine decay to \texttt{lr\_max}/20, and gradient clipping $1.0$. AdamMini and AdamMini‐BSFA share this schedule. For AdamMini‐BSFA, BPPE estimates a rank-$30$ subspace using $l=40$, $T=10$, and we set $\alpha=0.5, \gamma=1.0$ to mitigate loss spikes. All LLaMA experiments run on four NVIDIA RTX 4090 GPUs.

\subsection{Experiment Details in Section~\ref{sec:part3}}
\label{appn:part3}
\textbf{Figure~\ref{fig:main_results_a}: ViT‐S/16 on ImageNet‐1k.} We evaluate BSFA on training a ViT‐S/16 model on ImageNet‐1k. We adopt the experimental protocol established by \citet{beyer2022betterplainvitbaselines}. Specifically, we adopt the original \texttt{timm} implementation and apply RandAugment and Mixup (level $10$, probability $0.2$). The default optimizer is AdamMini with hyperparameters $\beta_1 = 0.9$, $\beta_2 = 0.999$, and weight decay $\lambda = 10^{-4}$ \citep{beyer2022betterplainvitbaselines}. We set the batch size to $1024$ and train for $90$ epochs (or $45$ epochs for AdamW‐BSFA), each run including $10{,}000$ warm‐up steps. The learning‐rate schedule comprises a linear warm‐up to \texttt{lr\_max} = $10^{-3}$ followed by cosine decay. For AdamW‐BSFA, we retain the same configuration and train for $45$ epochs. We use BPPE to estimate the dominant subspace with $k=50$, $l=60$, and $T=10$, and tune $\alpha\in\{1,0.5,0.2\}$, $\gamma\in\{2,4,6\}$ via grid search. All ViT‐S experiments run on a single NVIDIA RTX 4090 GPU.%Representative hyperparameter pairs are shown in Figure~\ref{fig:vit_hparams}. 

\textbf{Figure~\ref{fig:main_results_b}: LLaMA(72M) on Wikitext‐103.} In this experiment we evaluate the acceleration of BSFA when applied to AdamMini. The model and task are identical to Appendix~\ref{appn:part2}. We optimize with AdamW using $\beta_1 = 0.9$, $\beta_2 = 0.95$, and weight decay $\lambda = 0.1$, and run for $50{,}000$ or $100{,}000$ steps. The learning‐rate schedule includes a $500$‐step warm‐up followed by cosine decay from \texttt{lr\_max} to \texttt{lr\_min} = \texttt{lr\_max}/20, with gradient clipping at norm $1.0$. We tune \texttt{lr\_max} for AdamW over $\{2\times10^{-4},\,4\times10^{-4},\,6\times10^{-4},\,8\times10^{-3},\,1.0\times10^{-3}\}$ as shown in Figure~\ref{fig:tuning72M}.

\begin{figure}[h]
  \includegraphics[width=\columnwidth]{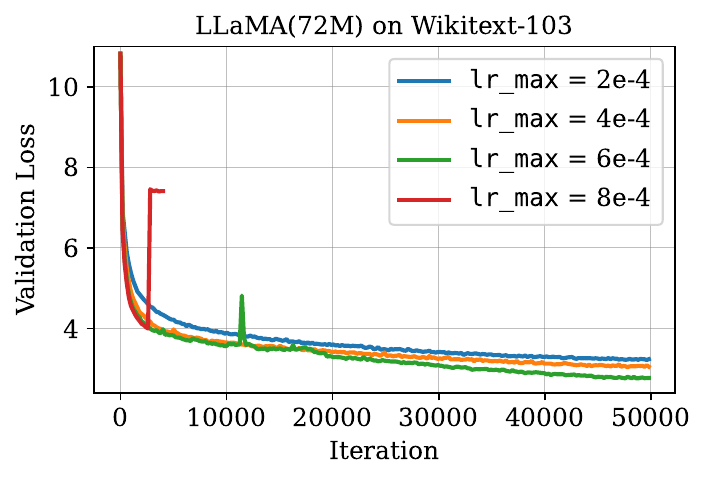}
  \caption{Tuning \texttt{lr\_max} for LLaMA(72M) on Wikitext-103.}
    \label{fig:tuning72M}
\end{figure}

\begin{table*}[t]
  \centering
  \begin{tabular}{@{}lccccc@{}}
    \toprule
    Network\&Task & $k$ & $l$ & $||\mathcal{D}||$ & LPE Time & \textbf{PPE Time} \\
    \midrule
    Transformer\&SST2       & 2  & 10  & 1000 & 10.28s  & \textbf{0.12s ($\bm{\downarrow~98.83\%}$)} \\
    Resnet18\&CIFAR10       & 10 & 20  & 5000 & 162.1s  & \textbf{0.42s ($\bm{\downarrow~99.74\%}$)} \\
    ViT-Tiny\&ImageNet-1k   & 20 & 50  & 5000 & 2002s   & \textbf{3.27s ($\bm{\downarrow~99.84\%}$)} \\
    \bottomrule
  \end{tabular}
  \caption{Average wall-clock time (in seconds) on a single RTX4090. Here, \( k \) denotes the number of eigenvectors(For both algorithms), \( l \) indicates the number of historical gradients(For PCA estimator), and \(|| \mathcal{D}|| \) represents the size of the abridged dataset(For Lanczos estimator).}
  \label{tab:time_compare}
\end{table*}

For AdamMini we follow the partitioning of \citet{adam_mini_2024}; Figure~\ref{fig:main_results_b} compares AdamMini and AdamW under identical settings, indicating similar performance with AdamMini slightly behind. AdamW and AdamW‐BSFA share the same learning‐rate schedule. For AdamW‐BSFA, we use BPPE with $k=50$, $l=60$, and $T=10$, and tune $\alpha\in\{1,0.5,0.2\}$, $\gamma\in\{3,4,6\}$. Representative hyperparameter pairs are displayed in Figure~\ref{fig:tuningeff72M}. All LLaMA experiments use the Huggingface LLaMA implementation on four NVIDIA RTX 4090 GPUs.

\begin{figure}[h]
  \includegraphics[width=\columnwidth]{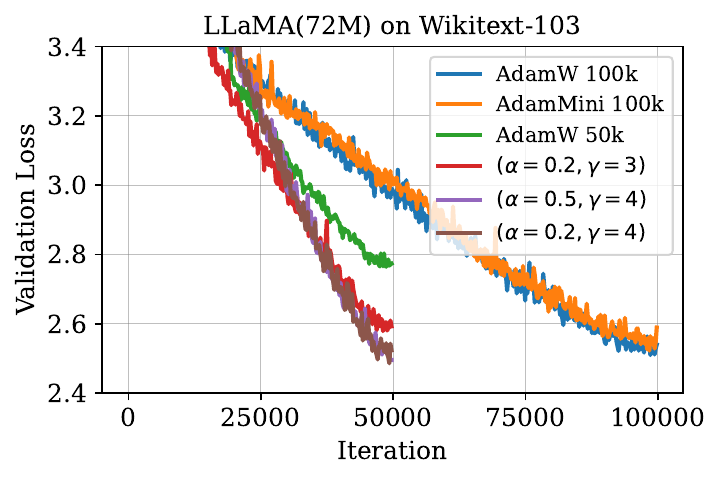}
  \caption{Tuning hyperparameter for AdamMini-BSFA.}
    \label{fig:tuningeff72M}
\end{figure}

\textbf{Figure~\ref{fig:main_results_c}: LLaMA(134M) on OpenWebText.} We evaluate BSFA on a 6-layer LLaMA(134M) model (16 heads per layer, hidden size 768) trained on the OpenWebText corpus, with maximum sequence length 1,024 and batch size 480 (following nanoGPT codebase\footnote{https://github.com/karpathy/nanoGPT} and \citet{liu2023sophia}). We optimize with AdamMini using the same hyperparameters of AdamW, where$(\beta_1,\beta_2,\lambda)=(0.9,0.95,0.1)$, run for 50,000 or 100,000 steps (including 1,000 warm-up steps), and apply gradient clipping (norm 1.0). The learning-rate schedule consists of a 1 000-step linear warm-up to \texttt{lr\_max} followed by cosine decay to \texttt{lr\_min} = \texttt{lr\_max}/20; we tune \texttt{lr\_max} over $\{3\times10^{-4},\,6\times10^{-4},\,1.2\times10^{-3},\,1.8\times10^{-3}\}$ and identify $6\times10^{-4}$ as optimal (The result is similar to Figure~\ref{fig:tuning72M}, so we don't show them again). AdamMini and AdamMini-BSFA share this schedule. For AdamMini-BSFA, BPPE estimates the dominant subspace with rank $k=50$, $l=60$ historical updates and interval $T=10$, and we grid-search $\alpha\in\{1,0.5,0.2\}$ and $\gamma\in\{2,4,6,8\}$ (The result is similar to Figure~\ref{fig:tuningeff72M}, so we don't show them again). All experiments use the Huggingface LLaMA codebase on four NVIDIA RTX 4090 GPUs.

\section{Details on Dominant Subspace Estimation}
\label{appen:dominant_estimation}

\subsection{Hessian Eigenspectrum Estimation via Lanczos Method}
\label{appen:lanczos}
Lanczos method~\citep{Pearlmutter1994,Papyan2018TheFS,pmlr-v97-ghorbani19b} is a common strategy for estimating the top-$k$ eigenpairs of the Hessian $\boldsymbol H$, which we refresh every $T$ optimizer steps to control overhead. In practice, this involves two phases:

\begin{enumerate}[leftmargin=10pt]
  \item \textbf{HVP construction.} On a small “abridged” dataset perform one forward pass and two backward passes to define the Hessian–vector product (HVP) operator $\boldsymbol v\mapsto\boldsymbol H\boldsymbol v$.
  \item \textbf{Lanczos iterations.} Starting from a random unit vector $q_1$, apply the HVP operator and orthogonalize against previous basis vectors to build a tridiagonal projection. Specifically, for $k=1,2,\dots$; do:
    \begin{align*}
      \tilde v_k &= \boldsymbol H\,q_k,\\
      \alpha_k    &= q_k^\top \tilde v_k,\\
      r_k         &= \tilde v_k - \alpha_k\,q_k - \beta_{k-1}\,q_{k-1},\\
      \beta_k     &= \|r_k\|,\\
      q_{k+1}     &= r_k / \beta_k.
    \end{align*}
\end{enumerate}

The resulting tridiagonal matrix can be diagonalized cheaply to yield approximations $\{\lambda_i,u_i\}_{i=1}^k$. Based on this, we propose the Lanczos-based Projector Estimator (LPE), summarized in Algorithm~\ref{algo:lanczos}.

\begin{algorithm}[H]
\caption{Lanczos-based Projector Estimator (\textbf{LPE})}
\label{algo:lanczos}
\begin{algorithmic}[1]
\State \textbf{Input:}
Neural network function \( \mathcal{F} \),Number of eigenvalues \( k \), Abridged Dataset \( \mathcal{D} \), Domspace scaler $\alpha$, Bulkspace scaler $\gamma$
\State \( \{\lambda_i\}_{i=1}^k, \{u_i\}_{i=1}^k \gets \text{Lanczos}( \mathcal{F}, \mathcal{D}, k) \)
\State $\mathcal{P}_{\alpha,\gamma} \gets \alpha\sum_{i=1}^{k} \boldsymbol{u}_i \boldsymbol{u}_i^{\top} + \gamma (\boldsymbol{I} - \sum_{i=1}^{k} \boldsymbol{u}_i \boldsymbol{u}_i^{\top})$
\State \textbf{Return:} \(\mathcal{P}_{\alpha,\gamma}\)
\end{algorithmic}
\end{algorithm}

\subsection{Wall-time Comparison of PPE and LPE}
\label{appen:time_compare}
We compare the runtimes of the PPE and LPE in Table~\ref{tab:time_compare}, which shows that the PPE's computation time is negligible compared to the LPE's. This is because the PPE does not involve the forward or backward propagation of the neural network.

  %\caption{Ablation study of BSFA across Transformer blocks. We compare three configurations: (a) BSFA applied only to chosen blocks with uniform amplification ($\gamma$) elsewhere; (b) BSFA restricted to chosen layers while other blocks use vanilla AdamW; (c) BSFA applied to either attention or MLP blocks alongside norm layers.}

\newpage
\onecolumn
\section{Proof of Proposition~\ref{prop:pca}}
\label{appen:proof}

First, we restate our proposition for readability.

\begin{proposition}[Gradient–PCA recovers the top eigenspace]
\label{prop:gd_pca_fixed}
Let $A\in\mathbb{R}^{d\times d}$ be symmetric and positive semi–definite with distinct
eigenvalues
\[
  \lambda_1>\lambda_2>\dots>\lambda_k>\lambda_{k+1}=\lambda_{k+2}=\dots=\lambda_d\;(\eqqcolon\lambda_{\mathrm{tail}}\ge 0),
\]
and corresponding orthonormal eigenvectors $v_1,\dots,v_d$.
Define the target subspace
$E_k=\mathrm{span}\{v_1,\dots,v_k\}$.

Choose a stepsize $\eta>0$ that satisfies
\begin{equation}
\label{eq:stepsize}
  \eta\lambda_k>1,\qquad
  0<\eta\lambda_{\mathrm{tail}}<1,\qquad
  \eta(\lambda_k+\lambda_{\mathrm{tail}})>2 .
\end{equation}

Run gradient descent on the quadratic
$f(x)=\tfrac12x^\top A x$ from an initial point
$x_0=\sum_{j=1}^d c_j v_j$ with $c_j\neq 0$ for $1\le j\le k$:
\[
  x_{s+1}=x_s-\eta A x_s,\qquad s=0,1,2,\dots
\]
and denote the gradients $g_s=\nabla f(x_s)=A x_s$.
For an integer window length $l\ge k$ form the data matrix
\[
  G_t=[\,g_t,\;g_{t+1},\;\dots,\;g_{t+l-1}\,]\in\mathbb{R}^{d\times l},\qquad t=0,1,2,\dots
\]

Then, as $t\to\infty$, the $k$ leading left singular vectors of
$G_t$ (equivalently, the top-$k$ eigenvectors of $G_tG_t^\top$)
converge to an orthonormal basis of $E_k$. For simplicity, we consider this uncentered case, i.e., a direct SVD analysis of the gradient matrix; we believe this conclusion can be naturally transferred to the centered version.
\end{proposition}

\begin{proof}

\textbf{1.  Closed forms.}
Write $x_0=\sum_{j=1}^d c_j v_j$.
Since $A v_j=\lambda_j v_j$, one step of gradient descent gives
$x_{s+1}=(I-\eta A)x_s$ then we have
\[
  x_s=(I-\eta A)^s x_0=\sum_{j=1}^d c_j \mu_j^{\,s} v_j,
\]
where $\mu_j\coloneqq 1-\eta\lambda_j$. This implies
\[
  g_s=A x_s=\sum_{j=1}^d \alpha_j\mu_j^{\,s} v_j,
\]
with $\alpha_j\coloneqq c_j\lambda_j$.

\textbf{2.  Properties of $\mu_j$.}
From the stepsize conditions \eqref{eq:stepsize}, we have
\[
  \mu_1<\mu_2<\dots<\mu_k<0,\qquad
  0<\mu_{k+1}=\dots=\mu_d<1.
\]
Since $\mu_k < 0$, we have $|\mu_k| = \eta\lambda_k - 1 > 0$. 
Since $\mu_{k+1} > 0$, we have $|\mu_{k+1}| = \mu_{k+1} = 1 - \eta\lambda_{\mathrm{tail}}$.

Critically, from the third condition $\eta(\lambda_k+\lambda_{\mathrm{tail}})>2$, we can derive:
\[
  \eta\lambda_k - 1 > 1 - \eta\lambda_{\mathrm{tail}}
\]
which means $|\mu_k| > \mu_{k+1}$. Combined with the order of $\mu_j$ values, we have:
\[
  |\mu_1|>|\mu_2|>\dots>|\mu_k|>\mu_{k+1}=|\mu_{k+1}|=\dots=|\mu_d|.
\]
This ensures that $|\mu_{k+1}|/|\mu_k|<1$, which is essential for convergence.

\textbf{3.  Factorization of $G_t$.}
Define $w_j\in\mathbb{R}^l$ and 
\(w_j=(1,\mu_j,\dots,\mu_j^{\,l-1})^\top\).
Then
\[
  G_t=\sum_{j=1}^d \alpha_j \mu_j^{\,t}\, v_j w_j^\top
      =V \Sigma_t W^\top,
\]
where $V=[v_1,\ldots,v_d]$, $\Sigma_t=\operatorname{diag}(\alpha_j\mu_j^{\,t})_{j=1}^d$, and 
$W=[w_1,\ldots,w_d]$. Consequently
\[
  G_tG_t^\top = V M^{(t)} V^\top,\quad\text{where}\quad
  M^{(t)} = \bigl(\alpha_i\alpha_j(\mu_i\mu_j)^{t}(w_i^\top w_j)\bigr)_{i,j=1}^d.
\]

\textbf{4.  Analysis of $M^{(t)}$.}
Write $M^{(t)}$ in block form
\[
  M^{(t)} = 
  \begin{pmatrix}
     M_{11}^{(t)} & M_{12}^{(t)}\\
     M_{12}^{(t)\top} & M_{22}^{(t)}
  \end{pmatrix}
\]
where $M_{11}^{(t)}\in\mathbb{R}^{k\times k}$. We can establish the following bounds:
\[
  \|M_{11}^{(t)}\|_2 \ge C_1 |\mu_k|^{2t},
\]
\[
  \|M_{12}^{(t)}\|_2 \le C_2 |\mu_k|^{t}\mu_{k+1}^{\,t},
\]
\[
  \|M_{22}^{(t)}\|_2 \le C_3 \mu_{k+1}^{\,2t},
\]
where $C_1, C_2, C_3 > 0$ are constants that depend only on $\{c_j,\lambda_j,l\}_{j=1}^d$. Specifically, $C_1$ depends on the minimum of $|\alpha_j|^2$ for $j\leq k$ and the inner products of the geometric progression vectors, while $C_2$ and $C_3$ depend on the maximum values of $|\alpha_i\alpha_j|$ and the corresponding inner products.

From these bounds and using eigenvalue perturbation theory, the eigen-gap between the $k$th and $(k+1)$th eigenvalues of $M^{(t)}$ satisfies:
\[
  \sigma_k(M^{(t)})-\sigma_{k+1}(M^{(t)}) \ge C_4 |\mu_k|^{2t}
\]
for some constant $C_4 > 0$, while
\[
  \|M_{12}^{(t)}\|_2 = O\bigl(|\mu_k|^{t}\mu_{k+1}^{\,t}\bigr).
\]

\textbf{5. Subspace Convergence via Davis--Kahan Theorem}
To establish the convergence of the computed subspace to $E_k$, we employ the Davis--Kahan $\sin\Theta$ theorem. This theorem bounds the difference between the subspace $S_t$ (derived from $M^{(t)}$) and the target subspace $E_k$, using the spectral properties of $M^{(t)}$:
\[
  \|\!\sin\Theta(S_t,E_k)\|_2
  \;\le\;
  \frac{\|M_{12}^{(t)}\|_2}
       {\sigma_k(M^{(t)})-\sigma_{k+1}(M^{(t)})}
  \;=\;
  O\!\Bigl(\bigl|\tfrac{\mu_{k+1}}{\mu_k}\bigr|^{\!t}\Bigr)
  \;\xrightarrow[t\to\infty]{}\;0,
\]
where $S_t$ denotes the span of the top-$k$ eigenvectors of
$M^{(t)}$ (equivalently of $G_tG_t^\top$).
Because \(V\) is orthogonal, the corresponding subspace in the
original coordinates is generated by the first $k$ left singular
vectors of \(G_t\).
Therefore \(S_t\to E_k\) as claimed.
\end{proof}

\end{document}